\documentclass[12pt, draftclsnofoot, onecolumn]{IEEEtran}
\usepackage{float}
\usepackage{stackengine}
\usepackage{setspace}
\usepackage{amssymb,amsfonts,latexsym}
\usepackage{graphicx}
\usepackage{subfigure}
\usepackage{psfrag}
\usepackage{algorithmic}
\usepackage{blindtext}
\usepackage{comment}
\usepackage{cite}
\usepackage{url}
\usepackage{times}
\usepackage{bbm}
\usepackage{epsfig}
\usepackage{multirow}
\usepackage{longtable}
\usepackage{amsfonts}
\usepackage{amssymb}%
\usepackage{color}
\usepackage{mathrsfs}
\usepackage{bm}
\usepackage{booktabs}
\usepackage{threeparttable}
\usepackage[cmex10]{amsmath}
\usepackage{amsmath}
\usepackage[english]{babel}
\usepackage{mathtools, nccmath}
\usepackage[linesnumbered,ruled]{algorithm2e}
\makeatother
\usepackage[table]{xcolor}
\usepackage{makecell}
\usepackage{stackengine}

\usepackage{stackengine}

\usepackage{stfloats}

\newtheorem{theorem}{\textbf{Theorem}}
\newtheorem{proof}{\textbf{Proof}}

\newtheorem{remark}{Remark}

\ifCLASSINFOpdf
\else
\fi

\begin{document}
	\title{Dynamic RAN Slicing for Service-Oriented Vehicular Networks via Constrained Learning }
	\author{\small Wen Wu,~\IEEEmembership{\small Member,~IEEE,}
	Nan Chen,~\IEEEmembership{\small  Member,~IEEE,}
	Conghao Zhou,~\IEEEmembership{\small Student~Member,~IEEE,}\\
	Mushu Li,~\IEEEmembership{\small Student~Member,~IEEE,}	        
	Xuemin~(Sherman)~Shen,~\IEEEmembership{\small Fellow,~IEEE},
	Weihua~Zhuang,~\IEEEmembership{\small Fellow,~IEEE}, and
	Xu~Li
	\thanks{W. Wu, N. Chen, C. Zhou, M. Li,  X. Shen, and W. Zhuang are with the Department of Electrical and Computer Engineering, University of Waterloo, Waterloo, ON,  N2L 3G1, Canada (e-mail:\{w77wu, n37chen, c89zhou, m475li, sshen,  wzhuang\}@uwaterloo.ca). \emph{Corresponding author: Mushu Li.}  }
	\thanks{X. Li is with Huawei Technologies Canada Inc., Ottawa, ON, K2K 3J1, Canada (email:xu.lica@huawei.com).}	
}

\maketitle
	\vspace{-0.5cm}
\thispagestyle{empty}
\begin{abstract}

	In this paper, we investigate a radio access network (RAN) slicing problem for Internet of vehicles (IoV) services with different quality of service (QoS) requirements, in which multiple logically-isolated slices are constructed on a common roadside network infrastructure. A dynamic RAN slicing framework is presented to dynamically allocate radio spectrum and computing resource, and distribute computation workloads for the slices. To obtain an optimal RAN slicing policy for accommodating the spatial-temporal dynamics of vehicle traffic density, we first formulate a constrained RAN slicing problem with the objective to minimize long-term system cost. This problem cannot be directly solved by traditional reinforcement learning (RL) algorithms due to complicated \emph{coupled constraints} among decisions. Therefore, we decouple the problem into a resource allocation subproblem and a workload distribution subproblem, and propose a \emph{two-layer constrained} RL algorithm, named \underline{R}esource \underline{A}llocation and \underline{W}orkload di\underline{S}tribution  (RAWS) to solve them. Specifically, an \emph{outer layer} first makes the resource allocation decision via an RL algorithm, and then an \emph{inner layer} makes the workload distribution decision via an optimization subroutine. Extensive trace-driven simulations show that the RAWS effectively reduces the system cost while satisfying QoS requirements with a high probability, as compared with benchmarks.

\vspace{0.2 cm}
\begin{IEEEkeywords}
RAN slicing, constrained reinforcement learning, vehicular networks, coupled constraint, workload distribution.
\end{IEEEkeywords}

\end{abstract}


\section{Introduction}
Recent technology advances for vehicular communication networks have laid a solid foundation for diverse Internet of vehicles (IoV) services, e.g., autonomous driving~\cite{liang2017vehicular, lu2014connected}.  It is predicted that autonomous vehicles will represent 25\% of the automotive market, which is valued up to 77 billion US dollars by 2025~\cite{MicrosoftAVSurvey}. 
When a vehicle is on the road, a large number of computation-intensive tasks of IoV services are required to be processed. Processing such computation-intensive tasks by vehicles requires expensive on-board computing facilities and degrades fuel efficiency~\cite{lin2018architectural}. As a remedy to these limitations, a potential solution is to explore the multi-access edge computing (MEC) paradigm~\cite{zhang2019mobile, xu2018joint}, in which vehicles can offload these computation tasks to computation-powerful roadside radio access networks (RANs) for prompt processing.

IoV services are diversified with different quality of service (QoS) requirements. For example, a delay-sensitive cooperative sensing service for autonomous driving has a stringent delay requirement, e.g., 100\;\emph{ms}~\cite{zhang2019mobile}; while a resource-hungry high definition (HD) map creation service is delay-tolerant; and a mobile video streaming service for vehicle users requires high throughput~\cite{qiao2018improving}.  To support these diversified IoV services with different QoS requirements, the promising RAN slicing approach for vehicular networks has emerged, which aims at constructing multiple logically-isolated slices on the roadside network infrastructure~\cite{shen2020ai}. 

In the literature, there exist some works investigating RAN slicing in the context of vehicular networks. Campolo~\emph{et~al.} propose a conceptual RAN slicing framework to support various vehicle-to-everything services~\cite{campolo20175g}. Another work presents a radio spectrum slicing scheme for roadside content caching services~\cite{zhang2020hierarchical}. While the existing works focus on resource allocation among slices, computation workload distribution is yet to be considered in the RAN slicing. {When base stations (BSs) are densely deployed along the road, vehicles are able to access the computing resources and offload their computation tasks to multiple BSs. Due to spatially uneven vehicle traffic density, some BSs may be overloaded. Vehicles in the coverage of the overloaded BSs can offload computation tasks to other underutilized BSs for better service performance, i.e., \emph{workload distribution}. Hence, a comprehensive RAN slicing policy should jointly consider resource allocation and  workload distribution.}

%

Developing an optimal RAN slicing policy  encounters many challenges. {Firstly}, the radio spectrum and computing resource allocation and workload distribution decisions are \emph{coupled}. The resource allocation decision determines the resource availability at BSs, thus affecting the workload distribution decision. In turn, the workload distribution decision yields a new task assignment between vehicles and BSs, consequently affecting the resource allocation decision. Efficient RAN slicing policy depends on the interaction between workload distribution and resource allocation. {Secondly}, IoV services are highly heterogeneous in terms of  QoS requirements, hence RAN slicing is required to satisfy multiple \emph{constraints}. {Thirdly}, since the system operates in the presence of  spatial-temporal dynamics of vehicle traffic density, maximizing long-term performance requires future information of vehicle traffic density. However, RAN slicing decisions have to be made without \emph{a priori} information. To address the above challenges, reinforcement learning (RL) is a potential approach to  make proactive decisions via learning traffic dynamics~\cite{Li2020deep}. 
	Directly applying traditional RL algorithms cannot solve the complicated RAN slicing problem with \emph{coupled constraints}, since RL algorithms make the resource allocation and workload distribution decisions separately and may violate the coupled constraints. Hence, a tailored RL algorithm is required for the RAN slicing problem with coupled constraints. 

In this paper, we present a dynamic RAN slicing framework for vehicular networks, in which workload distribution and resource allocation decisions are made for IoV services with different QoS requirements.  Considering spatial-temporal variations of vehicle traffic density,  RAN slicing is formulated as a stochastic optimization problem to minimize the long-term overall system cost while satisfying coupled constraints. To solve the problem, firstly, we decouple the problem into a resource allocation subproblem and a workload distribution subproblem. The resource allocation subproblem aims at minimizing the long-term system cost, which can be reformulated as a Markov decision process (MDP). The workload distribution subproblem is further decomposed into multiple one-shot optimization problems with the objective of minimizing instantaneous service delay, which are proved to be convex. Secondly, leveraging the properties of two subproblems, we propose a \emph{two-layer constrained} RL algorithm, named \underline{R}esource \underline{A}llocation and \underline{W}orkload di\underline{S}tribution (RAWS). First, the \emph{outer layer} of the RAWS makes the resource allocation decision via an RL algorithm, and then the \emph{inner layer} makes the workload distribution decision via a convex optimization subroutine, thereby satisfying the coupled constraints. 
Extensive trace-driven simulation results demonstrate that the RAWS can effectively reduce the overall system cost, {as compared with the state-of-the-art RL benchmarks including the deep deterministic policy gradient (DDPG) algorithm~\cite{lillicrap2015continuous} and the twin delayed DDPG (TD3) algorithm~\cite{fujimoto2018addressing}.} Particularly, the RAWS can satisfy QoS constraints with a high probability, even in a heavy traffic scenario.
	
The main contributions of this paper are summarized as follows:
\begin{itemize}
	\item  We present a dynamic RAN slicing framework for vehicular networks to support multiple IoV services and balance BSs' workloads. We formulate the RAN slicing  as a constrained stochastic optimization problem. The objective is to dynamically make resource allocation and workload distribution decisions to minimize the long-term overall system cost while satisfying coupled constraints and resource capacity constraints;
	\item 
 We decouple the constrained stochastic optimization problem into a resource allocation subproblem and a workload distribution subproblem, and propose a two-layer constrained RL algorithm to solve it. We also design a softmax-based actor network within the algorithm to generate the resource allocation decision satisfying resource capacity constraints that the total amounts of allocated resources cannot exceed resource capacities at each BS.
\end{itemize}


The remainder of this paper is organized as follows. Related works are reviewed in Section~\ref{sec:related_work}. {The system model, the RAN slicing framework, and the problem  formulation are presented in Section~\ref{sec:system_model}.} The two-layer  RAN slicing algorithm is proposed in Section~\ref{sec:two_layer_algorithm}. Simulation results are given in Section~\ref{sec:simulations}. Finally, Section~\ref{sec:conclusions} concludes this research work.

\section{Related Work}\label{sec:related_work}

Network slicing can be divided into core network slicing and  RAN slicing. Different from core network slicing which has been widely investigated, RAN slicing is still in its infancy, which has garnered much attention from both industry and academia recently. In the industry, the third generation partnership project (3GPP) devotes to standardizing RAN slicing in the fifth generation network~\cite{3GPP2017}.  Recent standardization efforts are discussed in~\cite{afolabi2018network}. In addition, several proof-of-concept RAN slicing systems have been developed, e.g., Orion~\cite{foukas2017orion}. In the academia, there are a number of research works investigating resource allocation problems in the  RAN slicing. Ye \emph{et al.} propose a communication resource slicing scheme for a two-tier cellular network~\cite{ye2018dynamic}, in which radio spectrum is allocated to a delay-oriented machine-to-machine service and a throughput-oriented data service. Xiao \emph{et al.} study computing resource allocation among slices in a fog computing framework~\cite{xiao2018dynamic}. Li \emph{et al.} propose a hierarchical soft RAN slicing framework to enable resource sharing among BSs, which can enhance resource multiplexing gain~\cite{li2020hierarchical}.  Another important line of works focuses on RAN slicing in vehicular networks. A  RAN slicing architecture is proposed for supporting diverse IoV services in~\cite{zhuang2019sdn}, in which a software defined networking (SDN) controller jointly orchestrates communication, caching and computing resources. Within this architecture, a soft RAN slicing framework is proposed in~\cite{zhang2020hierarchical} to support multiple roadside caching services, in which network resources of slices can be opportunistically reused to enhance resource multiplexing gain. {Different from the existing works that focus on resource allocation among slices, this work deals with the impact of spatially uneven vehicle traffic density on the system performance. We propose a workload distribution mechanism for BSs' workload balance.}



Advanced machine learning algorithms have been widely applied in the network slicing, such as slice resource demand prediction \cite{gutterman2019ran}, virtual network function deployment~\cite{pei2019optimal}, and traffic flow scheduling~\cite{gu2019intelligent}. Very recently, RL-based algorithms are applied to deal with complicated resource allocation problems in  RAN slicing~\cite{chen2019multi}. In~\cite{roig2019management}, a deep RL algorithm is proposed to dynamically deploy virtual network functions by adjusting computing and storage resources. Xiang \emph{et al.} study the problem of content caching and mode selection for multiple RAN slices, in which an RL algorithm is proposed to address the challenges of time-varying channel and the unknown content popularity information~\cite{xiang2020realization}. Different from the existing studies, our work focuses on a \emph{constrained} RAN slicing problem targeting at satisfying coupled constraints, which cannot be directly solved by traditional RL algorithms. We propose a novel two-layer constrained RL algorithm by leveraging the property of the considered problem.

\section{{System Model and Problem Formulation}}\label{sec:system_model}
In this section, we first establish the network model under consideration, and then present a dynamic RAN slicing framework for vehicular networks. Within the framework, we derive QoS constraints for  the considered services. Finally, the overall system cost model is given to evaluate the performance of a RAN slicing policy. 

\subsection{Network Model}
{As shown in Fig. \ref{Fig:considered_scenario}, we consider an MEC-enabled RAN along a road segment consisting of $N$ BSs, each  with a circular coverage of radius $r$.} {The set of BSs is denoted by $\mathcal{N}$. }All the BSs are connected to an SDN controller, which is in charge of collecting network information and enforcing RAN slicing decisions. Each BS is equipped with a powerful computing server. Vehicles traversing the road segment can associate with a nearby BS and offload their computation tasks to BSs via wireless communication links. Then, BSs complete the computation tasks and send results to vehicles.  

\begin{figure}[t]
	\centering
	\renewcommand{\figurename}{Fig.}
	\includegraphics[width=0.4\textwidth]{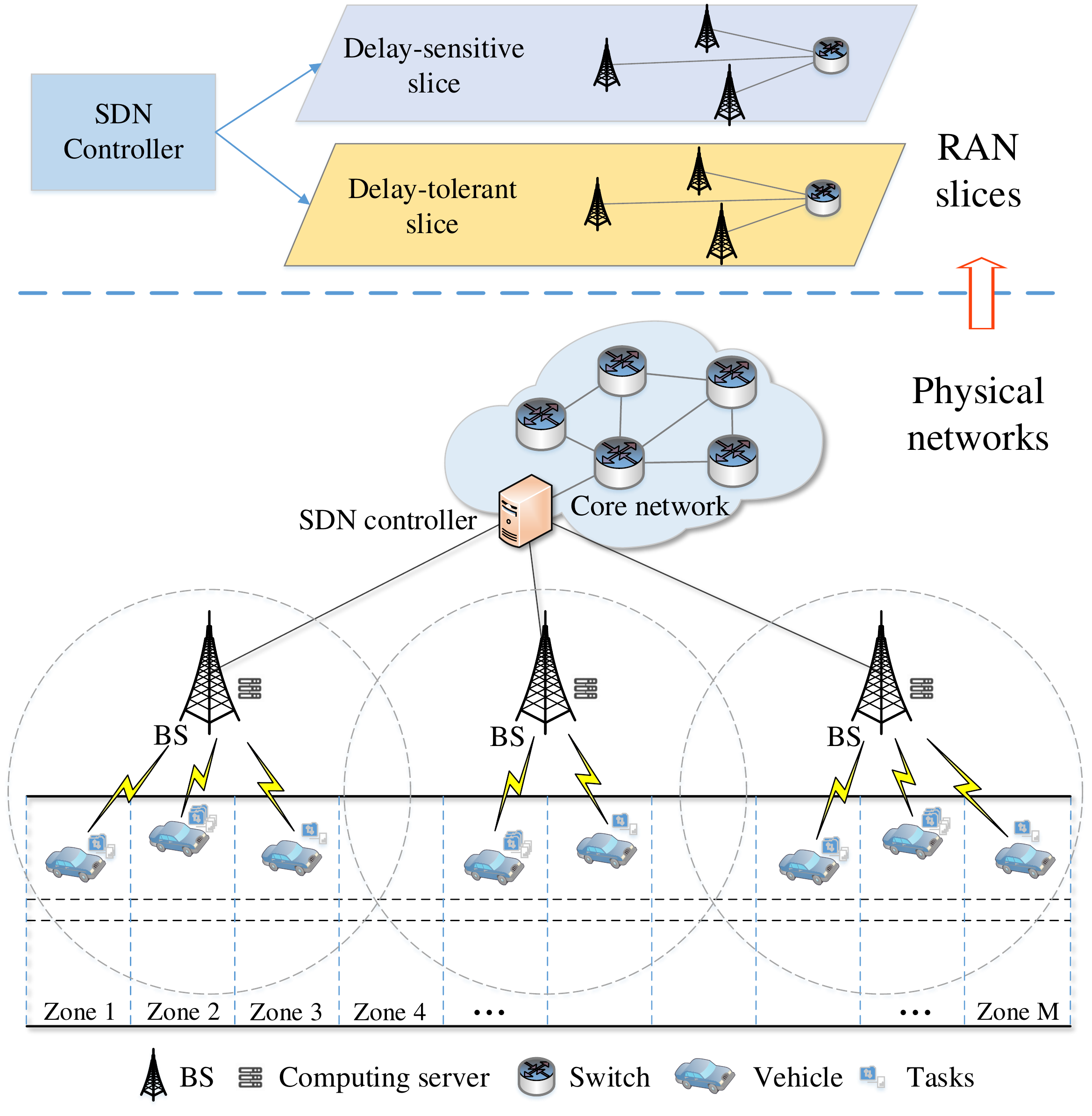}
	\caption{Dynamic RAN slicing framework for vehicular networks to support different IoV services.}
	\label{Fig:considered_scenario}
		\vspace{-0.3cm}
\end{figure}

We adopt a zone-based model to characterize spatially uneven vehicle traffic density along the road.  The considered road segment is divided into $M$ zones with equal length $L$, and BSs can identify these zones according to their locations. The set of zones is denoted by $\mathcal{M}$. The vehicle density of these zones is denoted by $\{\rho_1, \rho_2, ..., \rho_M\}$, and the corresponding average vehicle velocity of these zones is denoted by $\{v_1, v_2, ..., v_M\}$. {If a highway is considered, the relationship between the vehicle velocity and the vehicle density of a zone can be characterized via the fluid flow mobility model~\cite{tan2010analytical}, i.e., $v_m=v_f\left(1-{\rho_m}/{\rho_{max}}\right)$, where $v_f$ and $\rho_{max}$ denote the free flow speed and the maximum vehicle density of the road, respectively.}  Based on the availability of BSs, zones can be categorized into  two types: (1) \emph{non-overlapped zones} $\mathcal{M}_n$, which are  covered by a single BS, e.g., Zone~2 in Fig.~\ref{Fig:considered_scenario}. Let binary variable $c_{m, n}\in\{0,1\}$ denote the association pattern, where $c_{m, n}=1$ indicates that non-overlapped zone~$m$ is associated with BS~$n$. Here, we consider the aggregated computation workload within each zone, since the RAN slicing framework operates based on macroscopic vehicle traffic \cite{gutterman2019ran}. The workload of zone $m$ can only be distributed to BS~$n$; and (2) \emph{overlapped zones} $\mathcal{M}_o$, which are covered by multiple BSs, e.g., Zone~4 in Fig. \ref{Fig:considered_scenario}.  For the simplicity of analysis, vehicles in an overlapped zone can only associate with two nearest BSs. This assumption is reasonable since the channel condition deteriorates with the increase of connection distance. The association pattern is indicated by two binary variables, $a_{m,{n} }\in \{0, 1\}$ and $b_{m,{n}' }\in \{0, 1\}$, where  $a_{m,{n} }= 1$ and $b_{m,{n}' }= 1$ denote that zone $m$ is associated with BSs $\{{n}, {n}'\}$. The workload of the zone can be distributed to these two BSs. Let $\mathbf{B} \in \mathbb{R}^{|\mathcal{M}_o| \times  K}$ denote the \emph{workload distribution} decision. Each element, $\beta_{m,k}\in [0, 1]$, represents the fraction of the workload of service $k$ in zone $m$ that is distributed to BS~$n$. Correspondingly, the rest workload fraction, $1-\beta_{m,k}$, is distributed to BS $n'$. 

\subsection{Dynamic RAN Slicing Framework}
{Two types of IoV services, denoted by $\mathcal{K}=\{u, e\}$, are considered. Here, {$K=|\mathcal{K}|=2$ represents the number of the considered services}.  One is a delay-sensitive service, denoted by service~$u$, which has the maximum tolerable delay constraint. Taking the cooperative sensing service for example, the image captured by on-board cameras can be processed by BSs along with the sensing data from the roadside infrastructure. For road safety, the maximum tolerable delay of the cooperative sensing service is 100-150\;\emph{ms}~\cite {lin2018architectural}. The other service is delay-tolerant, denoted by service~$e$. Taking the HD map creation service for example, vehicles collect and upload fresh HD map data, and then BSs detect the changes from the fresh data and update the HD map~\cite{zhang2019mobile}.

{We present a dynamic RAN slicing framework to support the considered two services,\footnote{{The presented RAN slicing framework can be extended to the case with multiple delay-sensitive and delay-tolerant services, in which increasing the number of services is to increase the dimension of decisions.} } and mainly consider delay as the QoS metric of the considered services.} As shown in Fig.~\ref{Fig:considered_scenario}, two slices are constructed: (1) a delay-sensitive  slice, requires that service delay should be less than the maximum tolerable delay; and (2) a delay-tolerant slice, requires the stability of the queuing systems for computation tasks since they are offloaded and then processed. Detailed QoS constraints are presented for these two services in Subsection~\ref{subsec:QoS_model}. 

 The proposed RAN slicing framework operates in a time-slotted manner after the slices are constructed. Time is partitioned into multiple \emph{slicing windows}, indexed by $t\in\mathcal{T}= \{1,2,3, ..., T\}$. Within a slicing window,  RAN slicing decisions remain unchanged, and the vehicle traffic density is assumed to be stationary. {In each slicing window, the SDN controller: (1) obtains the average vehicle density  of all the zones in the current slicing window;\footnote{{The average vehicle density in the current slicing window can be predicted based on historical vehicle density information with high accuracy~\cite{han2019short}. The vehicle density can be collected via many methods, such as using roadside loop detectors~\cite{chen2001freeway} or on-board global positioning system (GPS) sensors.  }} (2) makes  workload distribution and  resource allocation decisions; and (3) evaluates the system performance based on the feedback from BSs at the end of the slicing window.} Specifically, the \emph{resource allocation} decision includes the allocation of  radio spectrum  and computing resources for slices at all the BSs. The radio spectrum resource is allocated in a unit of subcarrier with bandwidth $W$. Let $S^t_{n,u}$ and $S^t_{n,e}$ denote the number of subcarriers allocated to the delay-sensitive and delay-tolerant services at BS $n$, respectively. The computing resource is allocated in a unit of virtual machine (VM) instance with central processing unit (CPU) frequency~$F$~\cite{ghaznavi2015elastic}.  Let $C^t_{n,u}$ and $C^t_{n,e}$ denote the number of VM instances allocated to the delay-sensitive and delay-tolerant services at BS $n$, respectively. In summary, RAN slicing decisions can be represented by the following matrices:
\begin{equation}
\begin{split}
\mathbf{S}^t&= \left(  {S}^t_{n,k}\in \mathbb{Z}^+: n\in \mathcal{N}, k \in \mathcal{K}  \right), \\
\mathbf{C}^t&= \left(  {C}^t_{n,k}\in \mathbb{Z}^+: n\in \mathcal{N}, k \in \mathcal{K}  \right),\\
\mathbf{B}^t &=\left( \beta^t_{m,k} \in \left[0,1\right]: m \in \mathcal{M}_o, k\in \mathcal{K} \right),
\end{split}
\end{equation}
where $\mathbb{Z}^+$ represents the set of positive integers. Due to spatial-temporal variations of vehicle traffic density, the resource allocation and workload distribution decisions should be dynamically made in each slicing window. 

\subsection{QoS Constraints of Considered Services}\label{subsec:QoS_model}

In this subsection, we derive QoS constraints for the delay-sensitive  and  delay-tolerant services based on the queuing theory. A computation task is characterized by a three-parameter model $\{\xi_k, \eta_k, \lambda_k\}$, $\forall k \in \mathcal{K}$~\cite{sun2017emm}, where $\xi_k$ denotes the task data size (in bits), $\eta_k$ denotes the average task computation intensity (in CPU cycles per task), and  $\lambda_k$ denotes the task arrival rate per vehicle.  For notation simplicity, we omit $t$ in $\beta_{m,k}^t$, $\rho_m^t$, $C_{n,k}^t$, and $S_{n,k}^t$ in this subsection.


\subsubsection{Delay-sensitive service} 
The service delay consists of three parts: task offloading delay, task processing delay, and handover delay, as analyzed in the following.

Firstly, the task offloading delay represents the average time taken for offloading a task from a vehicle to  a BS. {The average transmission rate between BS $n$ and the covered zones is denoted by~${R}_n$.} Due to the stochastic wireless channel condition, the task transmission time from a vehicle to a BS is assumed to follow an exponential distribution with rate  $S_{n,u}{R}_n/\xi_u$~\cite{zhang2020hierarchical}. Task arrivals from each vehicle are assumed to follow a Poisson process~\cite{sun2017emm}, and hence the aggregated task arrivals also follow a Poisson process. The rate of task arrivals at BS~$n$ is $\chi_{n,u}+\sum_{m\in \mathcal{M}_o}\psi_{m,n, u}  \beta_{m,u}$, which consists of task arrivals  from the non-overlapped zones and overlapped zones. Here, $\chi_{n,u}=\sum_{m\in \mathcal{M}_n} c_{m,n} {\lambda}_u \rho_mL +\sum_{m\in \mathcal{M}_o}b_{m,n}{\lambda}_u\rho_m L $, and $\psi_{m,n,u} =\left(a_{m,n}-b_{m,n}\right) {\lambda}_u\rho_m L$ are constants, where ${\lambda}_u\rho_m L$ denotes the workload of zone $m$ of service $u$. The task offloading process can be modeled as an M/M/1 queue~\cite{ma2020cooperative}, and hence the average offloading delay at BS~$n$ is given by
\begin{equation}\label{equ:offloading_delay}
D_{n, u}^o =\frac{1}{S_{n,u}{R}_n/\xi_u-\chi_{n,u}-\sum_{m\in \mathcal{M}_o}\psi_{m,n, u}  \beta_{m,u}}, \forall n\in \mathcal{N}.
\end{equation}
Constraint
\begin{equation}\label{equ: constraint_U_1}
 S_{n,u}- \hat{\chi}_{n,s, u}- \kappa_{s,u}\sum_{m\in \mathcal{M}_o}\psi_{m,n,u}  \beta_{m,u} \geq 0, \forall n\in \mathcal{N}
\end{equation} 
holds to ensure the stability of the task offloading queue at BS $n$, which also indicates  radio spectrum allocation decision $ S_{n,u}$ and  workload distribution decision $ \beta_{m,u}$ are coupled. In \eqref{equ: constraint_U_1}, $\kappa_{s,u}={\xi}_u/{{R}_n}$ and $\hat{\chi}_{n,s,u}=\kappa_{s,u}\chi_{n,u}$ are constants.

Secondly, the task processing delay represents the average time of executing a task. The arrived tasks at the BS are placed in the computing queue until being processed. We assume the computation intensity per task follows an exponential distribution~\cite{xu2018joint}. Hence, the departure process is a Poisson process with rate $C_{n,u}F/\eta_u$. Similar to the analysis in the task offloading delay, the task processing delay at BS $n$ is given based on the M/M/1 queue theory, i.e.,
\begin{equation}\label{equ:processing_delay}
D^c_{n,u}=\frac{1}{C_{n,u}F/\eta_u-\chi_{n,u}-\sum_{m\in \mathcal{M}_o}\psi_{m,n,u}  \beta_{m,u}}, \forall n\in \mathcal{N}, 
\end{equation}
where  constraint 
\begin{equation}\label{equ: constraint_U_2}
C_{n,u}- \hat{\chi}_{n,c,u}- \kappa_{c,u}\sum_{m\in \mathcal{M}_o}\psi_{m,n,u}  \beta_{m,u}\geq 0, \forall n\in \mathcal{N} 
\end{equation} 
holds to ensure the stability of processing queue at BS $n$. In \eqref{equ: constraint_U_2}, $\kappa_{c,u}={\eta_u}/{F}$ and $\hat{\chi}_{n,c,u}=\kappa_{c,u}\chi_{n,u}$ are constants.

Thirdly, the average handover delay experienced by one task is defined as the ratio between the overall handover delay and the overall number of generated tasks of a vehicle in a road segment. The overall handover delay can be computed as the product of one-time handover delay and the number of handovers~\cite{ibrahim2016mobility}. The number of handovers can be deemed as the number of BSs. The overall number of generated tasks is the product of the task arrival rate and the sojourn time that a vehicle drives through the road segment, which is given by $\lambda_u\sum_{m\in \mathcal{M}}{L}/{v_m}$. Hence, the average handover delay is
\begin{equation}\label{equ: handover_delay}
D_{u}^{h}=\frac{D_{H}N}{\lambda_u\sum_{m\in  \mathcal{M}}{L}/{v_m}},
\end{equation}
where $D_{H}$ denotes the one-time handover delay.

With \eqref{equ:offloading_delay}, \eqref{equ:processing_delay} and \eqref{equ: handover_delay}, taking different BSs' workloads into account, the average service delay is given by
\begin{equation}\label{equ:delay_definition}
D_u=	D^h_{u}+ \sum_{n\in \mathcal{N}}\frac{\chi_{n,u}+\sum_{m\in \mathcal{M}_o}\psi_{m,n,u}  \beta_{m,u}}{\sum_{m\in \mathcal{M}}\lambda_u\rho_mL}\left(D^o_{n,u}+	D^c_{n,u}\right),
\end{equation} 
where $\sum_{m\in \mathcal{M}}\lambda_u\rho_mL$ denotes the amount of all the zones' workloads	. The service delay in slicing window $t$ should be less than the maximum tolerable delay $D^{th}$, i.e., $D_{u}^t\leq D^{th}$. {Similar to many existing works~\cite{sun2017emm, xu2018joint, chen2015efficient}, we assume that the downlink latency is negligible in the service delay, since the data size representing output computation results is much smaller than that of input computation tasks in various applications.} 


\subsubsection{Delay-tolerant service} 
The QoS constraint of the delay-tolerant service is to guarantee the queue stability, thereby avoiding queue overflow. The stability constraints are applied to the task offloading  and  task processing queues.
\begin{itemize}
	\item Stability of  task offloading queue: The analysis of the delay-tolerant service is similar to that of the delay-sensitive service in \eqref{equ: constraint_U_1} based on the queuing theory. 
	To guarantee the stability of the task offloading queue, constraint
	\begin{equation}\label{equ: constraint_B_1}
	S_{n,e} - \hat{\chi}_{n,s,e}-\kappa_{s,e}\sum_{m\in \mathcal{M}_o}\psi_{m,n,e}  \beta_{m,e}\geq 0, \forall n\in \mathcal{N}
	\end{equation}
	holds for  BS $n$. In \eqref{equ: constraint_B_1}, $\kappa_{s,e}={\xi_e}/{{R}_n}$, $\hat{\chi}_{n,s,e}=\kappa_{s,e}{\chi}_{n,e}$, ${\chi}_{n,e}=\sum_{m\in \mathcal{M}_n} c_{m,n}{\lambda}_e\rho_m L +\sum_{m\in \mathcal{M}_o}b_{m,n} {\lambda}_e\rho_m L$, and $\psi_{m,n,e} =\left(a_{m,n}-b_{m,n}\right) {\lambda}_e \rho_m L $ are constants. 
	\item Stability of  task processing queue: 
	Similar to the analysis for the delay-sensitive service in \eqref{equ: constraint_U_2}, the constraint on the allocated computing resource for the delay-tolerant service is given by
	\begin{equation}\label{equ: constraint_B_2}
		C_{n,e}- \hat{\chi}_{n,c,e}-\kappa_{c,e}\sum_{m\in \mathcal{M}_o}\psi_{m,n,e}  \beta_{m,e}\geq 0, \forall n\in \mathcal{N},
	\end{equation}
	where $\kappa_{c,e}={\eta_e}/{F}$ and $\hat{\chi}_{n,c,e}=\kappa_{c,e}\chi_{n,e}$ are constants.
\end{itemize}

\subsection{Overall System Cost Model}\label{section:overall_cost}
The overall system cost is defined to evaluate the performance of a RAN slicing policy, which includes operation cost, slice reconfiguration cost, delay constraint violation cost, and system revenue in each slicing window. 

\subsubsection{Operation cost}  
The operation cost refers to the cost of allocating subcarriers and VM instances for  slices. Let $U_o^t$ denote the operation cost in slicing window $t$, given by
\begin{equation}\label{equ:Operation_Cost}
U_o^t=\sum_{n\in \mathcal{N}} \sum_{k\in\mathcal{K}} \left( w_{o,s}S_{n,k}^t+ w_{o,c}C_{n,k}^t\right).
\end{equation}
{In \eqref{equ:Operation_Cost}, $w_{o,s}$ and $w_{o,c}$ represent the unit costs of using a subcarrier and a VM instance, respectively, which should be set to be equivalent to account for radio spectrum and computing resource fairness. } 

\subsubsection{Slice reconfiguration cost}  
{When vehicle traffic density varies, the SDN controller may reconfigure slices and adjust the allocated network resources of slices, which renders the slice reconfiguration cost. Taking the computing resource adaption for example, VM adaption in practical systems (e.g., Docker~\cite{bernstein2014containers}) involves booting a new VM instance or resizing an instantiated VM instance, which incurs additional delay for preparing resources~\cite{roig2019management} or brings hardware wear-and-tear~\cite{wang2017online}. }Let $U_r^t$ denote the slice reconfiguration cost in slicing window $t$, which measures the difference of resource allocation decisions across two subsequent slicing windows, i.e.,
\begin{equation}\label{equ:Reconfiguration_Cost}
\begin{split}
U_r^t&=\sum_{n\in \mathcal{N}}\sum_{k\in \mathcal{K}} \left(w_{r,s}\left[S_{n,k}^{t+1}-S_{n,k}^t \right]^+\right.\\
&\left.+w_{r,c}\left[C_{n,k}^{t+1}-C_{n,k}^t \right]^+\right),
\end{split}
\end{equation}
	where function $\left[x^{t+1}-x^t\right]^+=\max \{x^{t+1}-x^t, 0\}$ is applied to capture the amount of the increased resources at the BS when transiting from slicing window $t$ to slicing window $t+1$. Note that the cost associated with reducing  resource is omitted, since resource release can be completed quickly with negligible cost~\cite{wang2017online}. {In \eqref{equ:Reconfiguration_Cost}, $w_{r,s}$ and $w_{r,c}$ represent the unit costs of  subcarrier and VM instance reconfiguration, respectively, which should take larger values than that of resource usage in order to discourage frequent slice reconfiguration.} {Note that the reconfiguration of workload distribution decision does not incur system cost. This is because workload distribution decision represents the association pattern between zones' workloads and BSs, which does not consume physical resources.} 

\subsubsection{Delay constraint violation cost} 
The delay constraint violation cost refers to the penalty once the service delay exceeds the maximum tolerable delay constraint, i.e.,
\begin{equation}\label{equ:QoS_constraint_violation_cost}
U_q^t= w_{v}  \mathbbm{1}\{D^t_{u}>D^{th}\},
\end{equation}
	where  $\mathbbm{1}{\{x\}}$ is an indicator function. If event $x$ is true, $\mathbbm{1}{\{x\}}=1$; otherwise, $\mathbbm{1}{\{x\}}=0$. {Here, $w_{v}$ denotes the unit cost for delay constraint violation, which should take an extremely large value to penalize constraint violation. }

\subsubsection{System revenue} 
The system revenue is related to the achieved service delay, which is given by
\begin{equation}\label{equ:revenue}
U_m^t= w_{r} \left[D^{th} - D^t_{u}\right]^+.
\end{equation}
{Here, $w_{r}>0$ is the unit revenue, which should take a relatively large value to encourage low-latency services. }
%
%
	
With \eqref{equ:Operation_Cost}-\eqref{equ:revenue}, the overall system cost in slicing window $t$ is given by
\begin{equation}\label{equ:system_cost}
U^t=U_o^t+ U_r^t+U_q^t-U_m^t.
\end{equation}
\begin{remark}
	The slice reconfiguration cost component belongs to \emph{dynamic cost} that measures the performance of RAN slicing decisions across slicing windows, while  the rest of  cost components belong to \emph{static cost} that measures the performance of  RAN slicing decisions in the current slicing window.  
\end{remark}


\subsection{{Problem Formulation}}


Since vehicle traffic density varies spatially and temporally, minimizing the overall system cost in the long-term while satisfying QoS constraints is paramount, especially from the perspective of the network operator. Hence, the dynamic RAN slicing problem is formulated as follows: 
\begin{subequations}\label{Problem 1}
	\begin{align}
	{\mathbf{P}_0:} \underset{\{\mathbf{S}^t, \mathbf{C}^t, \mathbf{B}^t\}_{t\in \mathcal{T}}}{\text{min}}\,\,
	& \mathbb{E}\left[\lim\limits_{T\to \infty} \frac{1}{T}\sum_{t=1}^{T} U^t \right]\nonumber \\
	 \text{s.t.}\,\,
	& \sum_{k\in \mathcal{K}}S_{n,k}^t\leq S_{n}^{\text{max}},\forall  n\in \mathcal{N} \label{equ:P0constraint_4}\\
	& \sum_{k\in \mathcal{K}} C_{n,k}^t\leq C_{n}^{\text{max}}, \forall n\in \mathcal{N} \label{equ:P0constraint_5}\\
	& S_{n,k}^t, C_{n,k}^t \in \mathbb{Z}^+,\forall  k\in \mathcal{K}, n\in \mathcal{N} \label{equ:P0constraint_6}\\
	& 0\leq \beta_{m,k}^t\leq 1,\forall   k\in \mathcal{K}, m\in \mathcal{M}_o \label{equ:P0constraint_7}\\
	& \eqref{equ: constraint_U_1},  \eqref{equ: constraint_U_2}, \eqref{equ: constraint_B_1}, \text{and } \space \eqref{equ: constraint_B_2},\nonumber
	\end{align}
\end{subequations}
where $S_{n}^{\text{max}}$ and $C_{n}^{\text{max}}$ denote the capacity of subcarriers and VM instances at BS $n$, respectively. The problem is to jointly allocate radio spectrum and computing resources, and distribute zones' workloads to minimize the average long-term system cost in an online manner.  Constraints \eqref{equ:P0constraint_4} and \eqref{equ:P0constraint_5} are \emph{resource capacity constraints}, which guarantee that the total amounts of allocated subcarriers and VM instances for all the services cannot exceed the capacity at each BS.  Constraints \eqref{equ:P0constraint_6} and \eqref{equ:P0constraint_7} guarantee the feasibility of resource allocation and workload distribution decisions.   Constraints \eqref{equ: constraint_U_1},  \eqref{equ: constraint_U_2}, \eqref{equ: constraint_B_1} and  \eqref{equ: constraint_B_2} ensure the stability of queues for all the services, which indicates resource allocation and workload distribution decisions are coupled, i.e., \emph{coupled constraints}.

Problem $\mathbf{P}_0$ belongs to stochastic optimization due to  spatial-temporal variations of vehicle density. {Due to the lack of future information, finding a global optimal solution via optimization methods is very difficult. To approach the optimal solution, one way is to use RL algorithms. Traditional RL algorithms can be applied to solve stochastic optimization problems with simple constraints, in which the space of decisions can be shaped separately according to their constraints. However, traditional RL algorithms make resource allocation and workload distribution decisions separately, which may violate the coupled constraints in problem $\mathbf{P}_0$. 
To  solve the complicated  problem with coupled constraints, we propose a two-layer constrained RL algorithm.}

\section{Two-layer Constrained RL Algorithm }\label{sec:two_layer_algorithm}

{In this section, we first decouple the problem into an outer-layer resource allocation subproblem and an inner-layer workload distribution subproblem, and then propose a two-layer constrained RL algorithm to solve them together. } 

\subsection{Inner Layer: Workload Distribution Subproblem}\label{subsec:workload_scheduling}
{The workload distribution subproblem is to minimize the long-term overall system cost via optimizing workload distribution decisions, i.e.,}
\begin{subequations}\label{Problem 1-long}
	\begin{align}
	{\mathbf{P}_{1}:}	 \underset{\{\mathbf{B}^t\}_{t\in \mathcal{T}}}{\text{min}}\,\,
	 & \mathbb{E}\left[\lim\limits_{T\to \infty} \frac{1}{T}\sum_{t=1}^{T} U^t \right] \nonumber \\
	 \text{s.t.}\,\,
	&  \eqref{equ: constraint_U_1},  \eqref{equ: constraint_U_2}, \eqref{equ: constraint_B_1},  \eqref{equ: constraint_B_2}, \text{and }  \eqref{equ:P0constraint_7}.\nonumber
	\end{align}
\end{subequations}
{According to the definition of the overall system cost  in \eqref{equ:system_cost},  workload distribution only impacts the static cost components, i.e., the delay bound violation cost and the system revenue, which are independent in each slicing window. Since the minimum value of the static cost components is achieved with the minimum service delay, the problem is to minimize the aggregated independent service delay in each slicing window. In addition,  constraints are independent in each slicing window. Hence, optimizing long-term optimization problem~$\mathbf{P}_{1}$ can  be converted  to individually optimizing multiple one-shot optimization problems, whose objectives are to minimize instantaneous service delay. }

{Moreover, each considered service has its own workload distribution decision variable, which can be optimized individually. For the delay-sensitive service, workload distribution variable ${\beta}_u^t=\{{\beta}_{m,u}^t\}_{m\in \mathcal{M}_o}\in \mathcal{R}^{|\mathcal{M}_o| \times 1}$ is optimized to minimize the average service delay in slicing window~$t$, which yields the following one-shot optimization problem:}
\begin{subequations}\label{Problem 1}
	\begin{align}
	{\mathbf{P}_{1,u}}:\,\, \underset{{\beta}_u^t}{\text{min}}\,\,
	&  D_u^t \nonumber \\
	\text{s.t.}\,\,
	&  \eqref{equ: constraint_U_1},  \eqref{equ: constraint_U_2}, \text{and }  \eqref{equ:P0constraint_7}.\nonumber
	\end{align}
\end{subequations}

\begin{theorem}\label{theorem}
	One-shot workload distribution problem $\mathbf{P}_{1,u}$ is a convex optimization problem.
\end{theorem}
\begin{proof}
	Proof is provided in Appendix~\ref{appendix:theorem}.
\end{proof}

{According to Theorem~\ref{theorem}, if the feasible region of problem $\mathbf{P}_{1,u}$ is non-empty, denoted by $\mathcal{F}_u\neq \emptyset$, the optimal workload distribution decision for the delay-sensitive service can be directly obtained via convex optimization solvers, such as CVX~\cite{grant2014cvx} and Gurobi~\cite{optimization2014inc}. However, the optimization problem can be infeasible if the allocated resource is insufficient to satisfy queue stability constraints in  \eqref{equ: constraint_U_1} and  \eqref{equ: constraint_U_2}. The infeasibility issue is solved by incorporating a penalty mechanism in the outer layer algorithm, which is detailed in Subsection~\ref{subsec: outer_layer}.}
	
{Similarly, for the delay-tolerant service, workload distribution decision ${\beta}_e^t=\{{\beta}_{m,e}^t\}_{m\in \mathcal{M}_o}$ is optimized only  to satisfy the queue stability constraints in \eqref{equ: constraint_B_1} and  \eqref{equ: constraint_B_2}. The corresponding problem, denoted by $\mathbf{P}_{1,e}$, is a simplified case of problem $\mathbf{P}_{1,u}$, which can also  be solved by convex optimization solvers. The feasible region of problem $\mathbf{P}_{1,e}$ is denoted by $\mathcal{F}_e$. In summary, when the resource allocation decision is given and optimization problems are solvable, the optimal workload distribution decision can be directly made via a convex optimization subroutine.}

{\textbf{Computational complexity analysis:} In order to use convex optimization solvers, such as CVX,  problem $\mathbf{P}_{1,u}$ is first transformed into a semi-definite programming (SDP) problem by introducing two sets of auxiliary variables. The computational complexity of solving the SDP problem by using the interior-point method is $\mathcal{O}\left(N_{ite}\left(3N-1\right)^{3.5}\right)$~\cite{ye2011interior}, where $N_{ite}$ represents the number of iterations before meeting the stopping criteria. Here, $3N-1$ is the number of variables in the transformed SDP problem, which includes workload distribution variable $\beta_u^t$ of dimension $N-1$, and auxiliary variables of dimension $2N$.  }

\subsection{Outer Layer: Resource Allocation Subproblem}\label{subsec: outer_layer}
The resource allocation subproblem is to allocate radio spectrum and computing resources with the objective of minimizing the average long-term system cost, i.e.,
\begin{subequations}\label{Problem 1}
	\begin{align}
	{\mathbf{P}_2:}	 \underset{\{\mathbf{S}^t, \mathbf{C}^t\}_{t\in \mathcal{T}}}{\text{min}}\,\,
	 &\mathbb{E}\left[\lim\limits_{T\to \infty} \frac{1}{T}\sum_{t=1}^{T} U^t \right]\nonumber \\
	 \text{s.t.}\,\,
	& \eqref{equ:P0constraint_4}, \eqref{equ:P0constraint_5}, \text{and } \eqref{equ:P0constraint_6}. 
	\end{align}
\end{subequations} 
The preceding problem is to design a resource allocation policy to minimize long-term system cost while satisfying constraints, which falls into the class of MDP with infinite horizon. 

In the following, we reformulate the resource allocation subproblem as an MDP. Specifically, the SDN controller is modelled as an \emph{agent}. In slicing window $t$, given observed \emph{state} $s^t$, the SDN controller takes resource allocation \emph{action}~$a^t$. Then, the corresponding \emph{reward} $r^t$ is fed back from the environment, and the state evolves into new state $s^{t+1}$ according to state
transition probability $\mathrm{Pr}\left(s^{t+1}|s^t, a^t\right)$~\cite{sutton2018reinforcement}. In the MDP, the action, state, and reward are defined as follows. 

\begin{itemize}
		\item{Action}: Corresponding to the optimizing variables in problem $\mathbf{P}_2$, the action is the radio spectrum and computing resource allocation decisions, i.e.,%
	\begin{equation}
	a^t=\left\{
	\{S_{n,k}^t\}_{    \begin{subarray}{c} n\in \mathcal{N}\\ k\in \mathcal{K} \end{subarray}  },
	\{C_{n,k}^t\}_{    \begin{subarray}{c} n\in \mathcal{N}\\ k\in \mathcal{K}	\end{subarray}  } 
	\right\} .
	\end{equation}
	Each action element takes a positive integer value in order to satisfy constraint \eqref{equ:P0constraint_7}. 

	\item State: The system performance depends on the  vehicle density of all the zones in the current slicing window $\{\rho_m^t\}_{m\in \mathcal{M}}$, and the resource allocation decision in the previous slicing window. Hence, the state is defined as
	\begin{equation}
s^t=\left\{ \{\rho_m^t\}_{m\in \mathcal{M}}, \{S_{n,k}^{t-1}\}_{    \begin{subarray}{c} n\in \mathcal{N}\\ k\in \mathcal{K} \end{subarray}  },\{C_{n,k}^{t-1}\}_{    \begin{subarray}{c} n\in \mathcal{N}\\ k\in \mathcal{K} \end{subarray}  }\right\} .
\end{equation} 
State space is  $\left\{\left(\mathbb{R}^+\right)^{M} \right\}\times \left\{ \left(\mathbb{Z}^+\right)^{KN}\right\} \times  \left\{\left(\mathbb{Z}^+\right)^{KN}  \right\} $, which is  continuous.

	\item{Reward}: 
{The following reward is defined to evaluate how good the action is under a state, i.e.,
	 \begin{equation}\label{equ:reward}
	r^t\left(s^t, a^t\right)=-\left(U^t \mathbbm{1} \left\{\mathcal{F}_u\neq \emptyset\right\}+w_f\sum_{k\in \mathcal{K}} \mathbbm{1} \left\{\mathcal{F}_k=\emptyset\right\}\right).
	\end{equation}
	In \eqref{equ:reward}, $w_f>0$ is the penalty weight, which should take an extremely large value. The reward consists of two terms. The first term is consistent with the objective function in problem $\mathbf{P}_2$, which is achieved when the workload distribution subproblem for the delay-sensitive service is feasible. Obviously, the reward is related to the average service delay in the current slicing window. The second term, $w_f\sum_{k\in \mathcal{K}} \mathbbm{1} \left\{\mathcal{F}_k=\emptyset\right\}$, is the penalty when the workload distribution subproblems are infeasible. Once the action (i.e., resource allocation decision) is determined, the workload distribution decision is made by solving the workload distribution subproblems in the inner layer. A poor resource allocation decision renders the infeasibility of the workload distribution subproblems, since insufficient network resources trigger the violation of queue stability constraints. To characterize the system performance in such a case,  a penalty is incurred to discourage poor resource allocation decisions. }
	

\end{itemize}





In the MDP setting, a resource allocation \emph{policy} specifies how the SDN controller allocates radio spectrum and computing resources according to the current network information in each slicing window. Let $\Pi$ denote the set of all the possible resource allocation policies. Our goal is to find the optimal resource allocation policy $\pi^\star \in \Pi$ that maximizes cumulative discounted reward within an infinite horizon of $T$ slicing windows, i.e.,
\begin{subequations}
	\begin{align}
{\mathbf{P}_2':} & \max_{\pi\in \Pi} \space 
&&\mathbb{E}\left[\lim_{T\rightarrow \infty}\sum_{t=1}^{T}(\gamma)^t  r^t \left(s^t, a^t\right) |\pi \right]\nonumber\\
& \text{s.t.}
&& \eqref{equ:P0constraint_4}, \text{and } \eqref{equ:P0constraint_5}, 
	\end{align}
\end{subequations}
where $\gamma\in (0,1)$ denotes the discount factor~\cite{sutton2018reinforcement}, and $(\gamma)^t$ denotes the $t$-th power of $\gamma$.  {The lacking of future information on the time-varying vehicle traffic density in vehicular networks results in the unknown state transition probability.} In addition to the unknown state transition probability, continuous state space further makes problem $\mathbf{P}_2'$ unsolvable via conventional model-based algorithms, such as value iteration and policy iteration approaches~\cite{sutton2018reinforcement}. Hence, a model-free RL algorithm which does not require state transition probability, can be applied to obtain an optimal resource allocation policy. The  algorithm design is detailed in Subsection~\ref{sec:algorithm}. {Note that the objective function in problem $\mathbf{P}_2$ with the cumulative undiscounted reward can be well approximated by the objective function in problem $\mathbf{P}'_2$  with the cumulative discounted reward, when the discount factor approaches 1~\cite{chen2017wireless,he2020edge}. }





%
%
\begin{figure}[t]
	\centering
	\renewcommand{\figurename}{Fig.}
	\includegraphics[width=0.5 \textwidth]{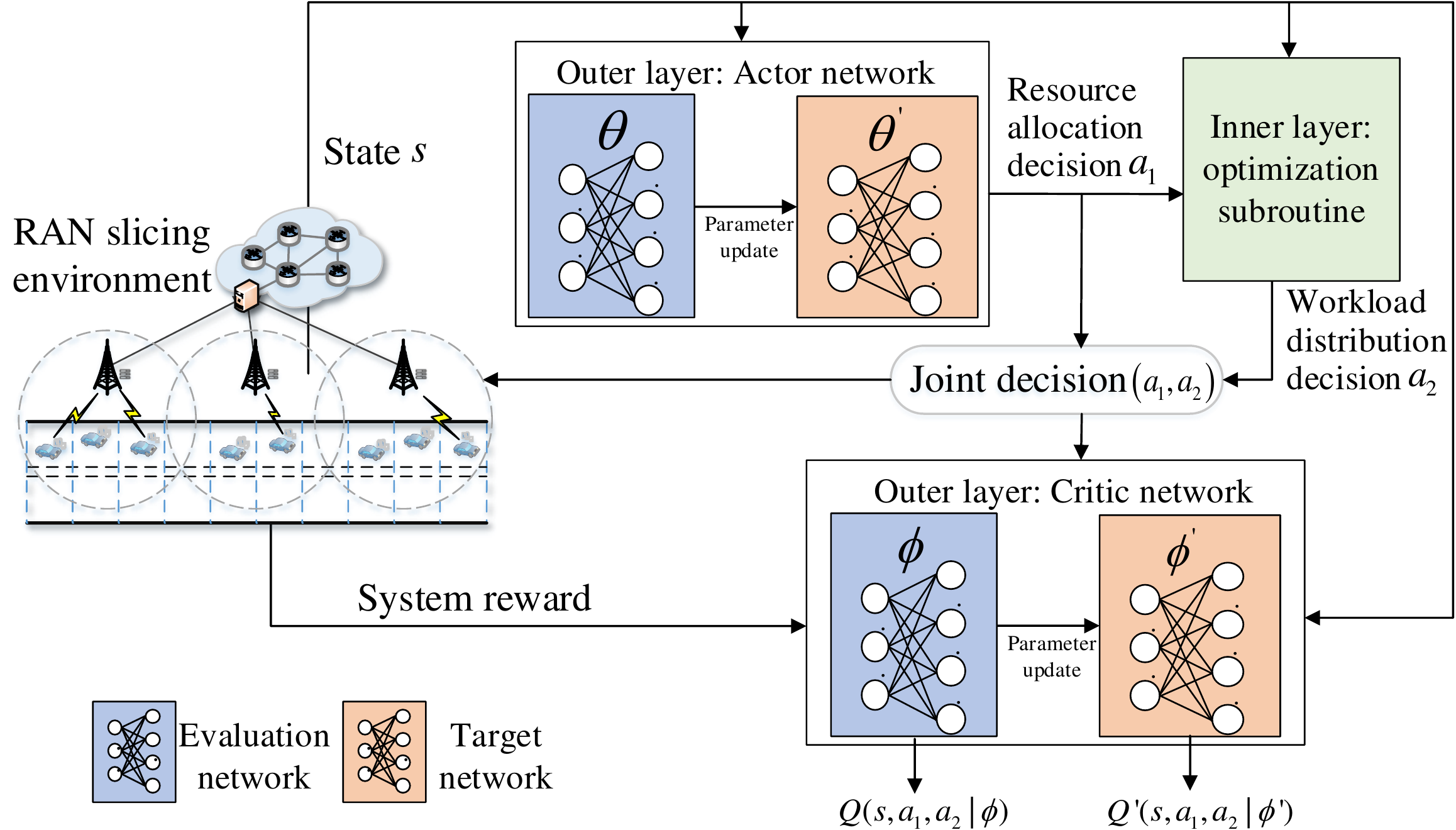}
	\caption{An illustration of the RAWS architecture. The actor network in the outer layer generates the resource allocation decision according to state~$s$, and the optimization subroutine in the inner layer generates the workload distribution decision  $a_2$ according to state $s$ and resource allocation decision~$a_1$. }
		\vspace{-0.3cm}
	\label{Fig:architecture}
\end{figure}


\subsection{RAWS Algorithm}\label{sec:algorithm}
By leveraging the properties of two subproblems, we propose a two-layer constrained RL algorithm to solve  problem~$\mathbf{P}_0$, named RAWS, which is extended from the DDPG algorithm~\cite{lillicrap2015continuous}. The two-layer architecture of RAWS is illustrated in Fig.~\ref{Fig:architecture}. The \emph{core idea} is that the outer layer and the inner layer make resource allocation and workload distribution decisions, respectively. Specifically, the actor network in the outer layer generates the resource allocation decision based on the current state, denoted by $a_1 $, while the optimization subroutine in the inner layer generates the workload distribution decision based on the current state and the actor network's action, denoted by $a_2$. In this way, a joint decision $(a_1, a_2)$ is obtained. The proposed decision making procedure complies with the properties of two subproblems. Then, a critic network evaluates the policy by estimating Q-values and guides the update of the actor network.   

Within the RAWS, we also design a softmax-based actor network to generate resource allocation decisions satisfying the resource capacity constraints, whose structure is shown in Fig.~\ref{Fig:actor_network}. The \emph{basic idea} is that the output layer of the actor network adopts the softmax function to activate a $(K+1)$-dimension vector, while only outputting its $K$ elements as the decision that satisfying the resource capacity constraint. The softmax activation function can obtain a normalized output vector, i.e., $\sum_{i=1}^{K+1}x_i=1$, where $\left\{x_i\right\}_{i=1,2,...,K+1}$ is the output vector. Hence, the sum of its $K$ elements satisfies  inequality constraint $\sum_{i=1}^{K}x_i\leq 1$, which can be scaled up to resource capacity constraints, e.g., $\sum_{k=1}^{K}S_{n,k}^t\leq S_{n}^{\text{max}}$. 

\begin{figure}[t]
	\centering
	\renewcommand{\figurename}{Fig.}
	\includegraphics[width=0.35\textwidth]{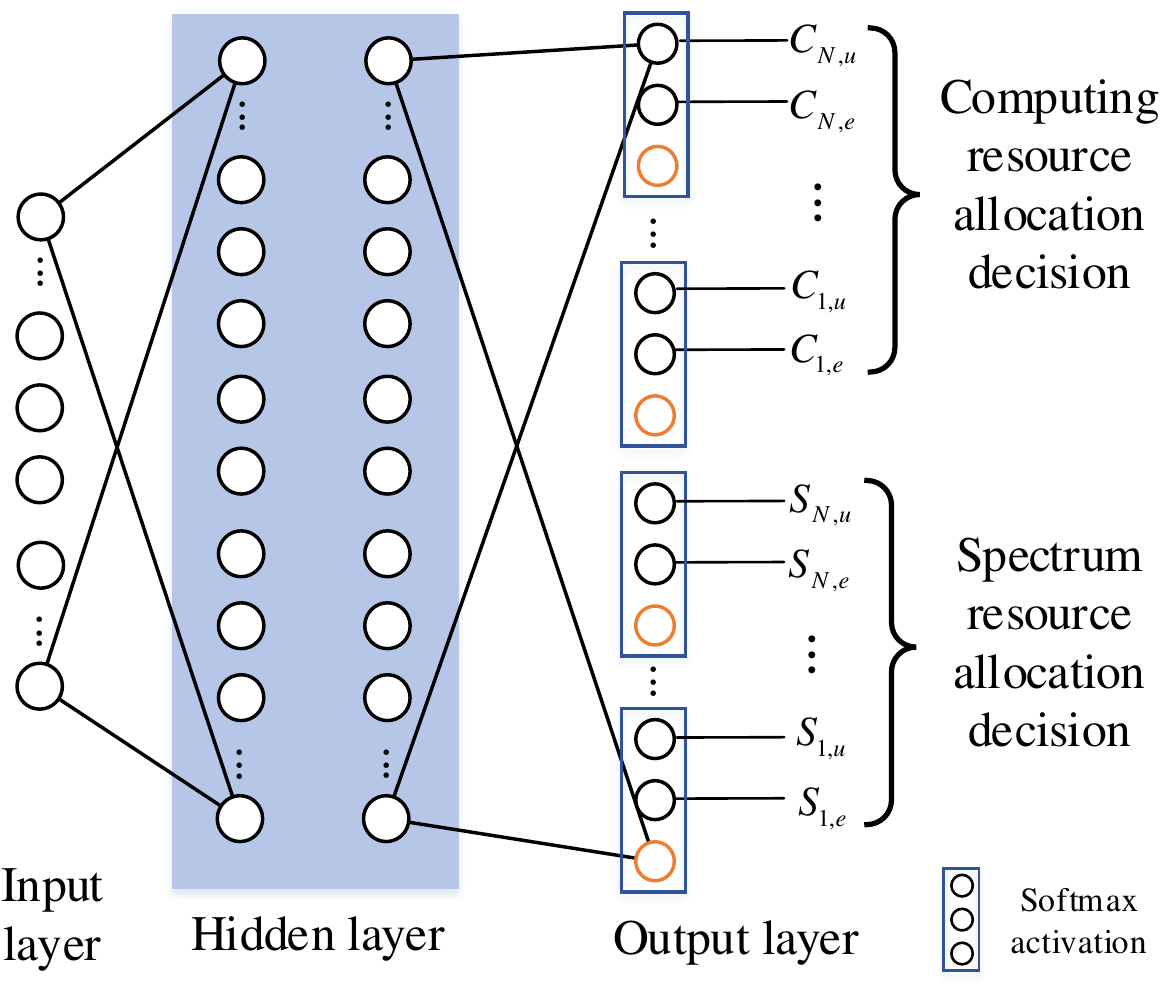}
	\caption{{Structure of the softmax-based actor network. Each tuple of three neurons in the output layer is activated by a softmax activation function, and the outputs of two of them are considered as the decisions.}}
	\label{Fig:actor_network}
		\vspace{-0.3cm}
\end{figure}

In the following, we elaborate the RAWS in detail, which is given in Algorithm 1. In the initialization phase, actor evaluation network $\mu(s|\theta)$, and critic evaluation network $Q(s, a_1, a_2|\phi)$ are initialized, where $\theta$ and $\phi$ denote parameters of evaluation networks. The corresponding target actor and critic networks are denoted by  $\mu'(s|\theta')$ and $Q'(s, a_1, a_2|\phi')$, respectively, whose parameters are $\theta'$ and $\phi'$. The environment is reset with  initialized state $s^0$. The RAWS operates in a discrete manner, which consists of two phases in each slicing window.

1. \emph{Experience generation (Lines 5-12)}: The agent's experience is represented by a  multi-dimensional tuple of the state transition and feedback reward, i.e., $\{s^t, a_1^t, a_2^t, r^t, s^{t+1}\}$, which is obtained via the following steps. The first step is to make the decisions.  In slicing window $t$, the actor network makes resource allocation decision $a_1^t=\{\mathbf{S}^t, \mathbf{C}^t\}$ with respect to current state~$s^t$, i.e., 
\begin{equation}\label{equ:joint_action1}
a_1^t=\mu(s^t|\theta)+\epsilon,
\end{equation}
where $\epsilon=\mathcal{N}(0, \sigma^2)$ is added Gaussian noise for policy exploration. Then, the workload distribution decision  $a_2^t=\mathbf{B}^t$ is made with respect to the current state and the resource allocation decision, by solving optimization problems $\left\{\mathbf{P}_{1,k}\right\}_{k\in\mathcal{K}}$ when they are feasible.\footnote{The feasibility of optimization problems can be obtained via the status of optimization solvers.} The second step is to store the experience. Once the joint decision is taken, reward $r^t$ is obtained from the environment, and the environment evolves to  new state $s^{t+1}$. The corresponding transition tuple $\{s^t, a_1^t, a_2^t, r^t, s^{t+1}\}$ is stored in experience replay buffer $B$ for training.

\begin{algorithm}[t]	\label{algorithm:DARS}
	\SetAlgoLined
	\LinesNumbered   
	\SetKwInOut{Input}{Input}
	\Input{ Vehicle density of all the zones $\{\rho_m^t\}_{m\in \mathcal{M}, t\in \mathcal{T}}$;}
	\SetKwInOut{Output}{Output}
	\Output{Resource allocation and workload distribution decisions $\{\mathbf{S}^t$, $\mathbf{C}^t$, $\mathbf{B}^t \}_{t\in \mathcal{T}}$;}
	\textbf{Initialization}: Initialize critic and actor networks and the experience replay buffer\;
	\For{episode =1: $N_e$}{
		Reset environment and obtain initial state $s^0$\;
		\For{slicing window $t$=1: $T$}{
			{$\rhd$ Experience generation}\\
			Obtain resource allocation decision $a_1^t$ by \eqref{equ:joint_action1}\;
			\If{$\left\{\mathbf{P}_{1,k}\right\}_{k\in\mathcal{K}}$ are feasible}{
				Obtain workload distribution decision $a_2^t$ by solving  optimization problems $\left\{\mathbf{P}_{1,k}\right\}_{k\in\mathcal{K}}$\;
			}
			\LinesNumbered   
			Execute joint decision $a^t=\left(a_1^t, a_2^t \right)$\;
			Observe reward $r^t$ and new state $s^{t+1}$\;
			Store transition $\{s^t, a_1^t, a_2^t, r^t, s^{t+1}\}$ in the experience replay buffer\;
			{$\rhd$ Neural network update}\\
			Sample a random minibatch of transitions from the experience replay buffer\;
			Update the critic evaluation network via minimzing the loss function in \eqref{equ:loss_function}\; 
			Update the actor evaluation network via the policy gradient in \eqref{equ: actor1_update}\;
			Update target networks by \eqref{equ:target_update}\;
		}
	}
	\caption{Resource Allocation and Workload Distribution (RAWS) algorithm.}
\end{algorithm}

2. \emph{Neural network update (Lines 13-17)}: In this phase, neural networks are updated offline based on the acquired experience, which is  similar to that in DDPG. At the beginning, a mini-batch of $N_m$ transitions are randomly chosen from the experience replay buffer as training data to update the parameters of neural networks. The detailed update procedure is given as follows. 

Firstly, the critic evaluation network is updated by minimizing the loss function $L(\phi)$, which is defined as
\begin{equation}\label{equ:loss_function}
L(\phi)=\frac{1}{N_m}\sum_{n=1}^{N_m}\left(y^n-Q\left(s^n, a_1^n, a_2^n|\phi\right) \right)^2.
\end{equation}
In \eqref{equ:loss_function}, $y^n=r^n+\gamma Q'\left(  s^n, a_1^n, a_2^n|\phi' \right)$ is the update target obtained from the target networks.

Secondly, guided by the critic network, the actor evaluation network is updated via the following policy gradient (i.e., the gradient of policy's performance):
\begin{equation}\label{equ: actor1_update}
\nabla_{\theta}J \approx \frac{1}{N_m}\sum_{n=1}^{N_m}\nabla_{a_1}{Q\left(s, a_1, a_2|\phi \right)|}_{\begin{subarray} {l}{s=s^n}\\ { a_1=\mu\left(s^n|\theta\right)} \end{subarray}} \nabla_{\theta}\mu\left(s^n|\theta\right).
\end{equation}

Thirdly, target networks are updated by slowly tracking the evaluation networks, i.e.,
\begin{equation}\label{equ:target_update}
\begin{split}
&{\theta'}=\tau {\theta}+(1-\tau) {\theta'},\\
&{\phi'}=\tau{\phi}+(1-\tau){\phi'},
\end{split}
\end{equation}
where $\tau \in (0, 1)$ denotes the update ratio of target networks.

	\begin{remark}
		Two unique designs are incorporated in the proposed RAWS algorithm: (1) the RAWS adopts a two-layer architecture, in which an actor network and an optimization subroutine are adopted to make joint decisions satisfying coupled constraints; and (2) the RAWS adopts a softmax-based actor network to generate resource allocation decisions satisfying resource capacity constraints. As such, the RAWS can reduce the long-term system cost while satisfying constraints.
\end{remark}

%

\section{Simulation Results}\label{sec:simulations}
\begin{figure}[t]
	\centering
	\renewcommand{\figurename}{Fig.}
	\includegraphics[width=0.5\textwidth]{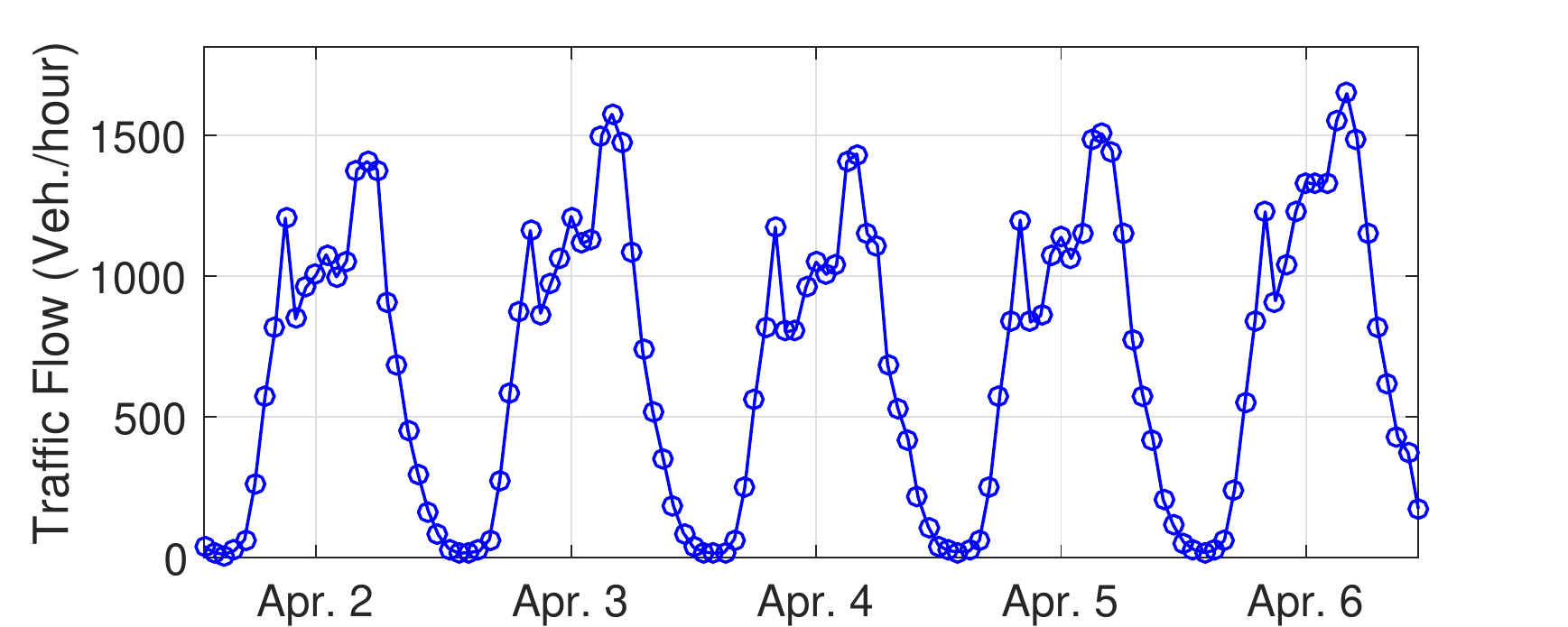}
	\caption{Highway vehicle traffic flow trace.}
	\label{Fig:vehicle_flow_trace}
		\vspace{-0.3 cm}
\end{figure}
\subsection{Simulation Setup}



The performance of the proposed RAWS algorithm is evaluated by simulations based on two real-world vehicle traffic flow traces: (1) the highway trace, which is collected by Alberta Transportation.\footnote{Alberta Transportation:  http://www.transportation.alberta.ca/mapping/.} The dataset records the hourly vehicle traffic flow on Highway 2A in Canada, over a period of three weeks from Apr. 2, 2020 to Apr. 22, 2020, as shown in Fig.~\ref{Fig:vehicle_flow_trace}.  The vehicle traffic flow information is converted to vehicle density based on the Lighthill-Whitham-Richards (LWR) model~\cite{kachroo1999feedback}; and (2) the urban trace, which is collected by Didi Chuxing GAIA Initiative,\footnote{{Didi Chuxing: https://gaia.didichuxing.com.}} over a period of three weeks from Oct. 10, 2016 to Oct. 30, 2016. The dataset collects the trace of taxis that are equipped with GPS devices in urban areas of Xi'an, China. The vehicle's location information is collected every 2-4 seconds, which can be converted to the average vehicle density information.  

\begin{table}[t]
\footnotesize
	\centering
	
	\caption{Simulation parameters.}
	\label{Simulation parameters}
	\vspace{-0.2cm}
	\begin{tabular}{ c c | c  c }
		\hline\hline
		\textbf{Parameter} &\textbf{Value} & \textbf{Parameter} & \textbf{Value} \\ \hline
		$W$ &  10 MHz  &$F$ & 10 GHz\\
		$\lambda_e$&$\{0.5, 1\}$ req/s &$\lambda_u$ &$\{0.8-1.2\}$ req/s \\
		$D^{th}$ & 100 \emph{ms}    &$D_H$ & 200 \emph{ms}\\
		$\rho_{max}$& 120 veh/km &$v_f$ &120 km/h\\ 
		$\xi_u, \xi_b$ & (0.6, 2) Mbit& $\eta_u, \eta_b$ & (6, 2) $\times10^8$  cycles \\
		$N_e$  & 1000 episodes &	$w_{o,s},w_{o,c}$ & 1, 1  \\
		$w_{r,s}$, $w_{o,c}$ & 5, 5 &$w_{v},w_{r}, w_f$ & 200, 25, 200 \\
		$S_n^{\text{max}}$ & 18     & $C_n^{\text{max}} $ & 18 \\ \hline
	\end{tabular}
	\vspace{-0.3 cm}
\end{table}

\begin{table}[t]
\footnotesize
	\centering
	\caption{Parameters of RAWS.}
	\label{Neural network parameters}
	\vspace{-0.2cm}
	\begin{tabular}{ c c|c c }
		\hline \hline
		\textbf{Parameters} & \textbf{Value} & \textbf{Parameters} & \textbf{Value}\\ \hline
		Actor learning rate & $10^{-4}$& Critic learning rate & $10^{-3}$ \\
		Actor  hidden units& $(128, 64)$& Critic hidden units& $(128, 64)$  \\ 
		Actor output act. & Softmax       	&Hidden layer act. & ReLU\\
		Optimizer& Adam & 	Replay buffer size& 100,000\\
		Policy noise  $(\sigma)$ &$0.02$  &Target update ratio & $0.005$\\
		Discount factor& $0.75$  &		 Minibatch size& $64$\\ \hline
	\end{tabular}
	\vspace{-0.3cm}
\end{table}

For both traces, the duration of a slicing window is set to one hour. We consider a road segment with a length of 5\;km, which is divided into~25 zones with equal length. Five BSs are deployed with a separation distance of 1\;km, and  the coverage radius of a BS is set to 800\;m for service continuity, in order to cover the road segment. The channel path loss model for vehicular networks is given by $\text{PL(dB)} = 128.1 + 37.6\log_{10}\left(d\right)$, where $d$ (in km) denotes the transmission distance between the vehicle and the BS~\cite{liang2019spectrum}. The transmission power is  set to 0.5\;W. The two types of services are considered in the simulation. For the delay sensitive service, we consider the cooperative sensing service for autonomous driving, in which vehicles upload on-board video to BSs. The corresponding task data size is set to 0.6\;Mbit, which is the data size of a one-second quarter common intermediate format (QCIF) video with 176$\times$144 video resolution, 24.8\;K pixels per frame and 25 frames per second~\cite{sun2017emm}. The average computation intensity is 1,000\;cycles/bit. After a task is processed, the  result is fed back to vehicles. Hence, the maximum tolerable delay of the cooperative sensing service is set to 100\;\emph{ms} for safety consideration~\cite{lin2018architectural}. For the delay-tolerant service, we consider the HD map creation service. The task data size and average computation intensity are set to 2\;Mbit and 100\;cycles/bit, respectively. The task arrival rate is set to 1~request per second unless otherwise specified. {{The exemplary unit value setting for the operation cost, the slice reconfiguration cost, the delay constraint violation cost, and the system revenue is $\{1, 5, 200, 25\}$ based on the mentioned principle in Subsection~III-D.} {{Note that these values can be tuned based on the preference of the network operator.}}}  Important system parameters are listed in Table~\ref{Simulation parameters}. Parameters of the RAWS are listed in Table~\ref{Neural network parameters}.


We compare the RAWS with the following  benchmarks:
\begin{enumerate}
	\item {\textbf{The DDPG~\cite{lillicrap2015continuous} with decision shaping}: The decision shaping is performed to obtain feasible decisions. The resource allocation decision is shaped to satisfy the coupled constraints, while keeping the workload distribution decision unchanged. Taking the constraint in \eqref{equ: constraint_U_1} for example, i.e., $S_{n,u}\geq 
		\hat{\chi}_{n,s, u}+ \kappa_{s,u}\sum_{m\in \mathcal{M}_o}\psi_{m,n,u}  \beta_{m,u}$, radio spectrum allocation decision $S_{n,u}$ is complemented once it violates the bound on the right-hand side;}  
	\item  \textbf{The TD3~\cite{fujimoto2018addressing} with decision shaping}: Similar to the DDPG, the TD3 algorithm shapes decisions according to coupled constraints;
	\item \textbf{The RAWS without workload distribution (RAWS w.o.)}: This algorithm is a simplified version of our proposed algorithm, in which the workload of an overlapped zone is equally distributed to  two nearby BSs;
	\item  \textbf{Random}: This algorithm randomly selects feasible decisions in the decision space. 
\end{enumerate}
Note that the softmax-based actor network is applied in all the learning-based algorithms to satisfy the resource capacity constraints.
\begin{figure*}[t]
	\centering
	\renewcommand{\figurename}{Fig.}	
	\begin{subfigure}[Overall system cost]{
			\label{Fig:overal_cost_trainiing}
			\includegraphics[width=0.4\textwidth]{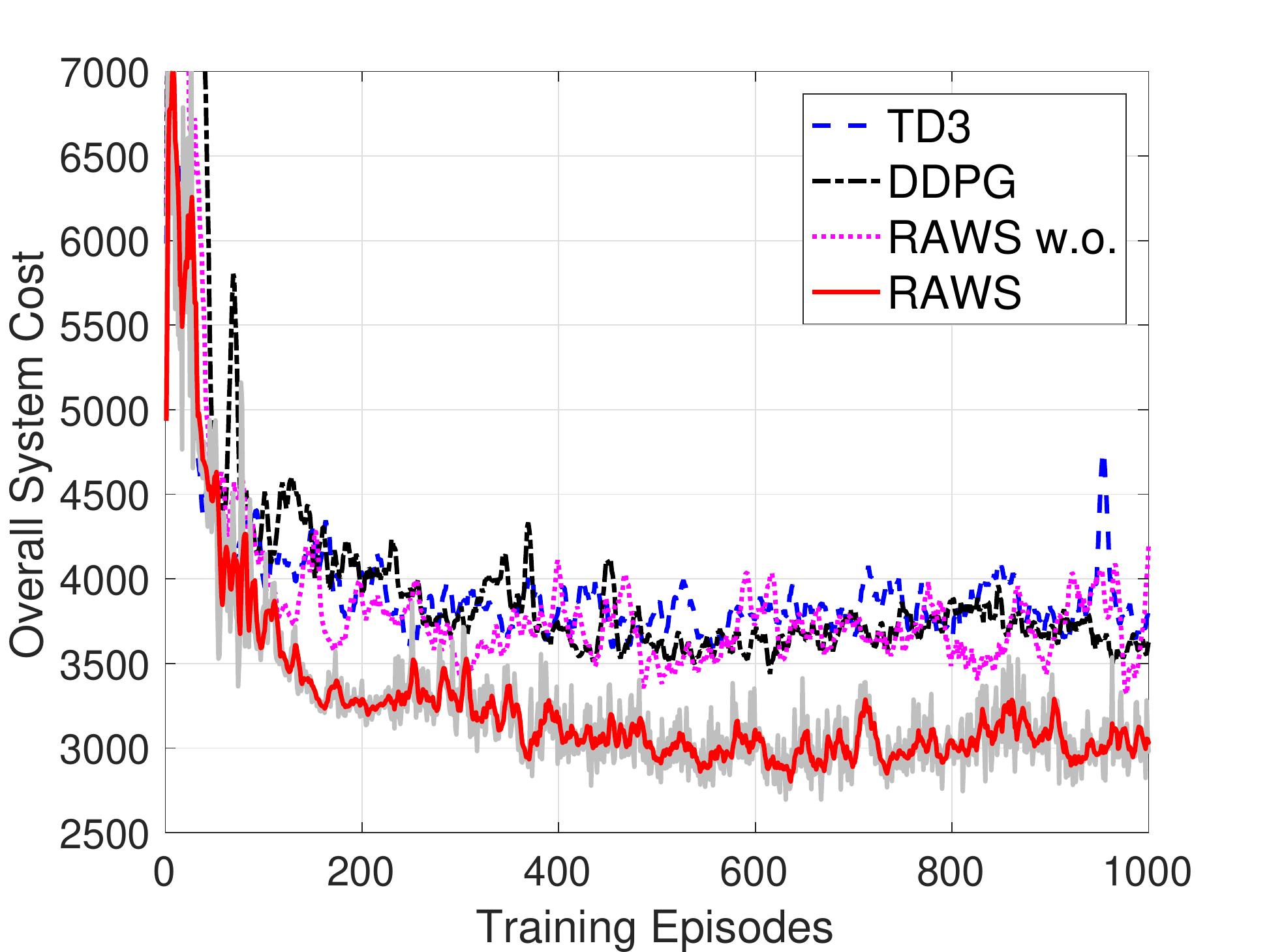}}
	\end{subfigure}
	~
	\begin{subfigure}[Delay-sensitive service performance ]{
			\label{Fig:delay_violation_training}
			\includegraphics[width=0.4\textwidth]{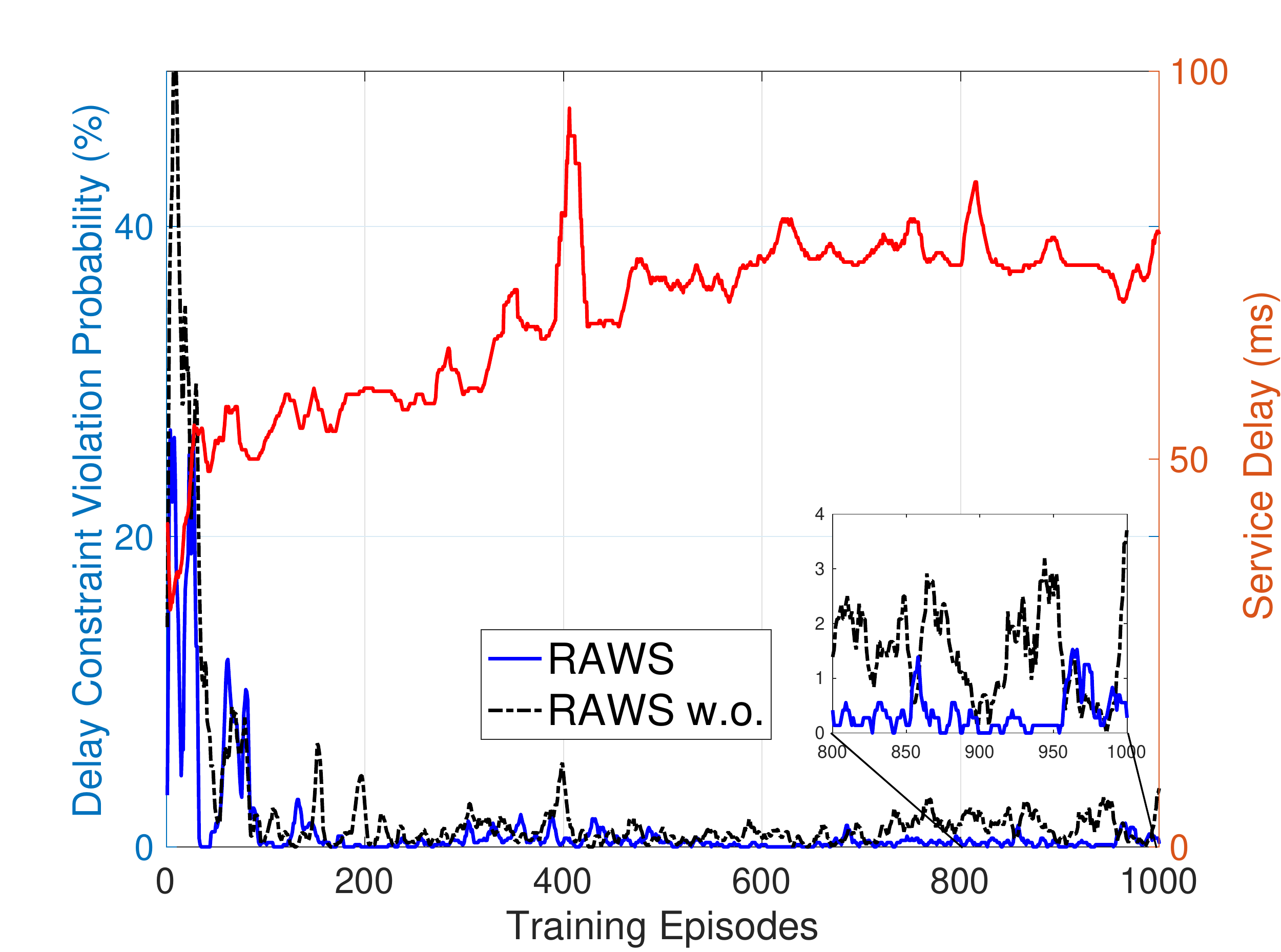}}
	\end{subfigure}
	\caption{Convergence performance of all the learning-based algorithms for $\lambda_u=1$. }
	\label{Fig:convergence_performance_comparsion}
	\vspace{-0.3cm}
\end{figure*}

\subsection{Performance Evaluation Over  Highway Vehicle Traffic Flow Trace}

{The first week's data is used to train the RAWS, and then the rest of two weeks' data is used to evaluate its actual performance in terms of system cost, service delay, and QoS constraint violation probability. The measured performance is averaged over two weeks. 
\subsubsection{Convergence performance}
The convergence performance of the RAWS is shown in Fig.~\ref{Fig:convergence_performance_comparsion}.  Raw simulation points (e.g., grey curve) are processed by a five-point moving average to highlight the convergence trend (e.g., red curve). {The number of training episodes represents the number of times that a algorithm has been trained.} Firstly, the overall system cost  with respect to training episodes is shown in Fig.~\ref{Fig:overal_cost_trainiing}.  We observe a ``bounded" overall system cost behavior of all the learning-based algorithms, which means that all of them have converged. More importantly, the RAWS achieves a lower overall system cost, as compared with all the benchmarks. The reason is that the RAWS is able to satisfy constraints without shaping decisions, while the DDPG and TD3 benchmarks do, thereby degrading the system performance.  Secondly, the delay-sensitive service performance with respect to training episodes is shown in Fig.~\ref{Fig:delay_violation_training}. On the one hand, the delay constraint violation probability quickly decreases from more than 20\% to approximately 0.6\% as the number of training episodes increases, which is lower than the benchmark. On the other hand, the red curve shows the average service delay also converges with respect to training episodes. Overall, the results indicate a good RAN slicing policy has been learned via interacting with a dynamic vehicular environment, which further verifies the convergence property of the RAWS.  }
\begin{figure}[t]
	\centering
	\renewcommand{\figurename}{Fig.}	
	\includegraphics[width=0.4\textwidth]{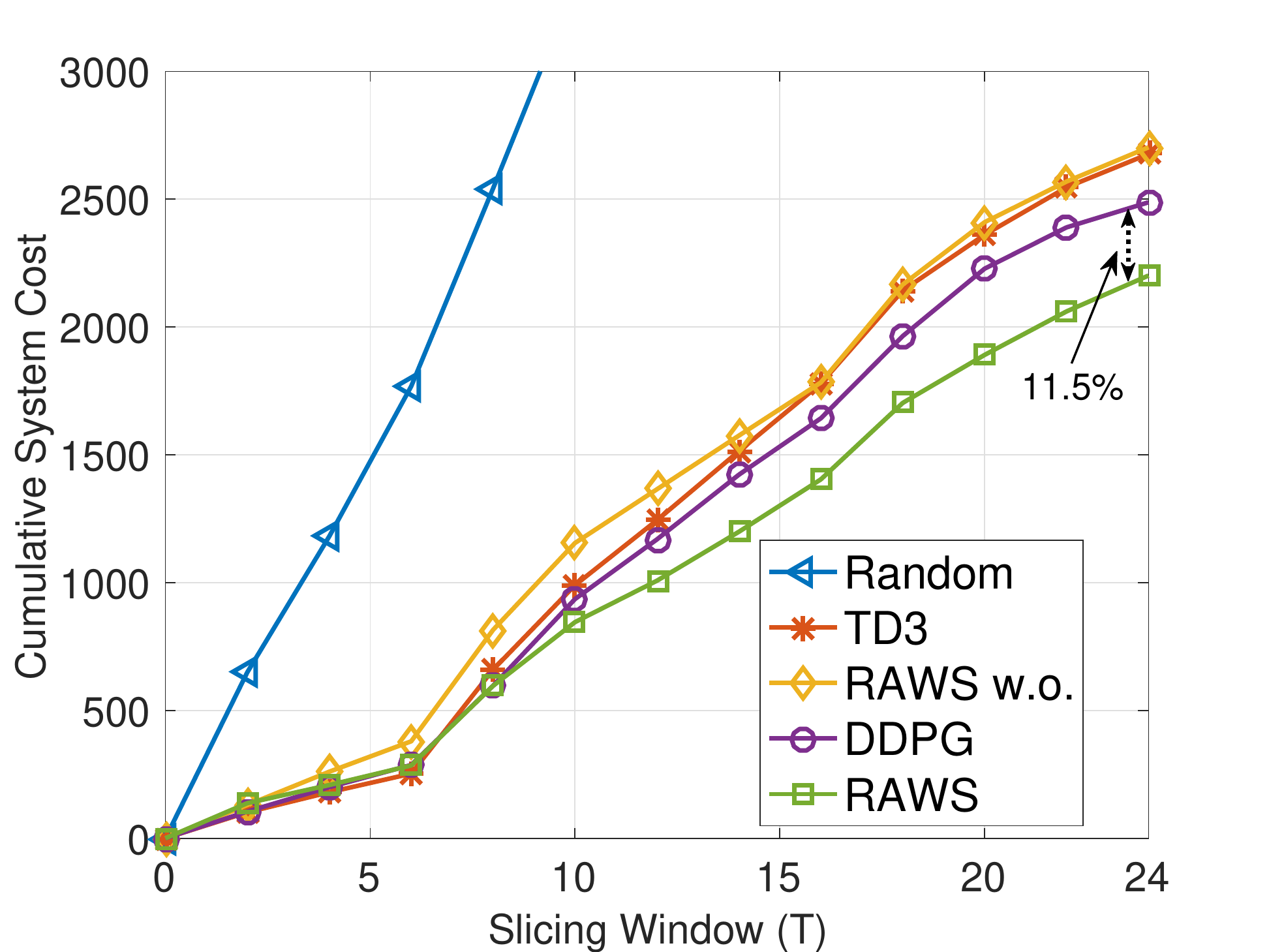}
	\caption{Cumulative system cost within one day for $\lambda_u=1$.}
	\label{Fig:aggregate_overall_cost}
	\vspace{-0.3cm}
\end{figure}



\subsubsection{System cost}
When the learning-based algorithms are well-trained, we evaluate the cumulative overall system cost $\sum_{t=1}^T U^t$ within one day, as shown in Fig.~\ref{Fig:aggregate_overall_cost}. In the simulation, one day includes 24 slicing windows for the slicing window size of one hour. As expected,  the RAWS incurs the lowest cost among all the algorithms. Specifically, the cumulative system cost incurred by the RAWS is approximately 11.5\% less than that by the best benchmark, which validates its good performance in reducing system cost. In addition, the result shows that the performance gain increases with the number of slicing windows, implying that a high performance gain can be achieved for a large number of slicing windows.

As shown in Fig.~\ref{Fig:overal_cost}, the average overall system cost per day of the different algorithms is compared with respect to the  task arrival rate of the delay-sensitive service. Obviously, the average system cost increases with the increase of the task arrival rate,  because more tasks consume more network resources. In addition, as compared with the benchmarks, performance gain achieved by the RAWS gradually increases with respect to the task arrival rate. The underlying reason is that the network resource utilization of  the RAWS is higher than that of the benchmarks, especially in heavy traffic scenarios (e.g., $\lambda_u=1.2$). In Fig.~\ref{Fig:operation_cost}, average operation cost component with respect to the task arrival rate is evaluated. We can see that the proposed RAWS algorithm incurs a lower operation cost. Specifically, average performance gain over different task arrival rates  is about 13\% for $\lambda_e=1$, as compared with the DDPG benchmark. 
\begin{figure*}[t]
	\centering
	\renewcommand{\figurename}{Fig.}	
	\begin{subfigure}[Overall system cost]{
			\label{Fig:overal_cost}
			\includegraphics[width=0.4\textwidth]{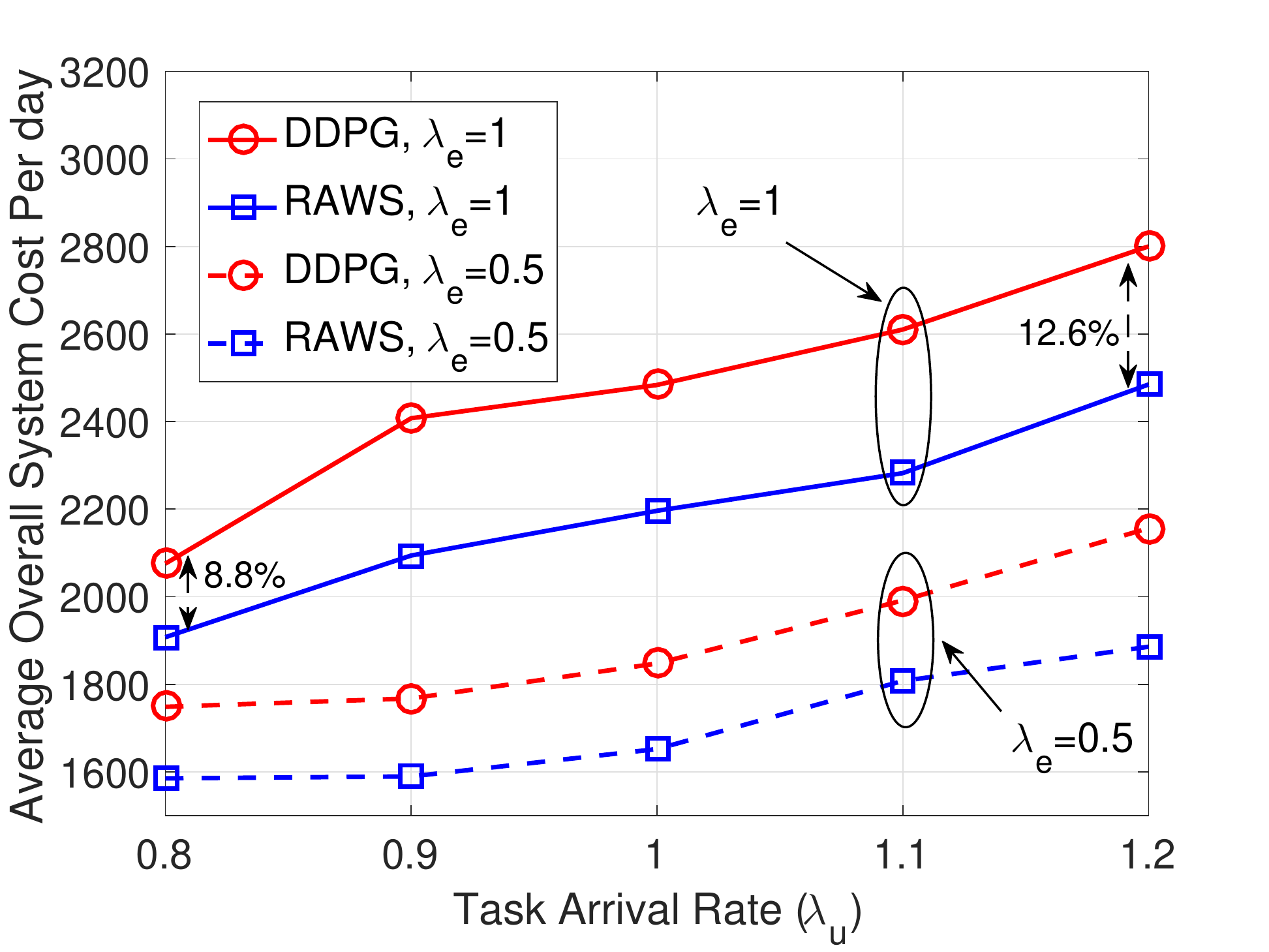}}
	\end{subfigure}
	~
	\begin{subfigure}[Operation cost]{
			\label{Fig:operation_cost}
			\includegraphics[width=0.4\textwidth]{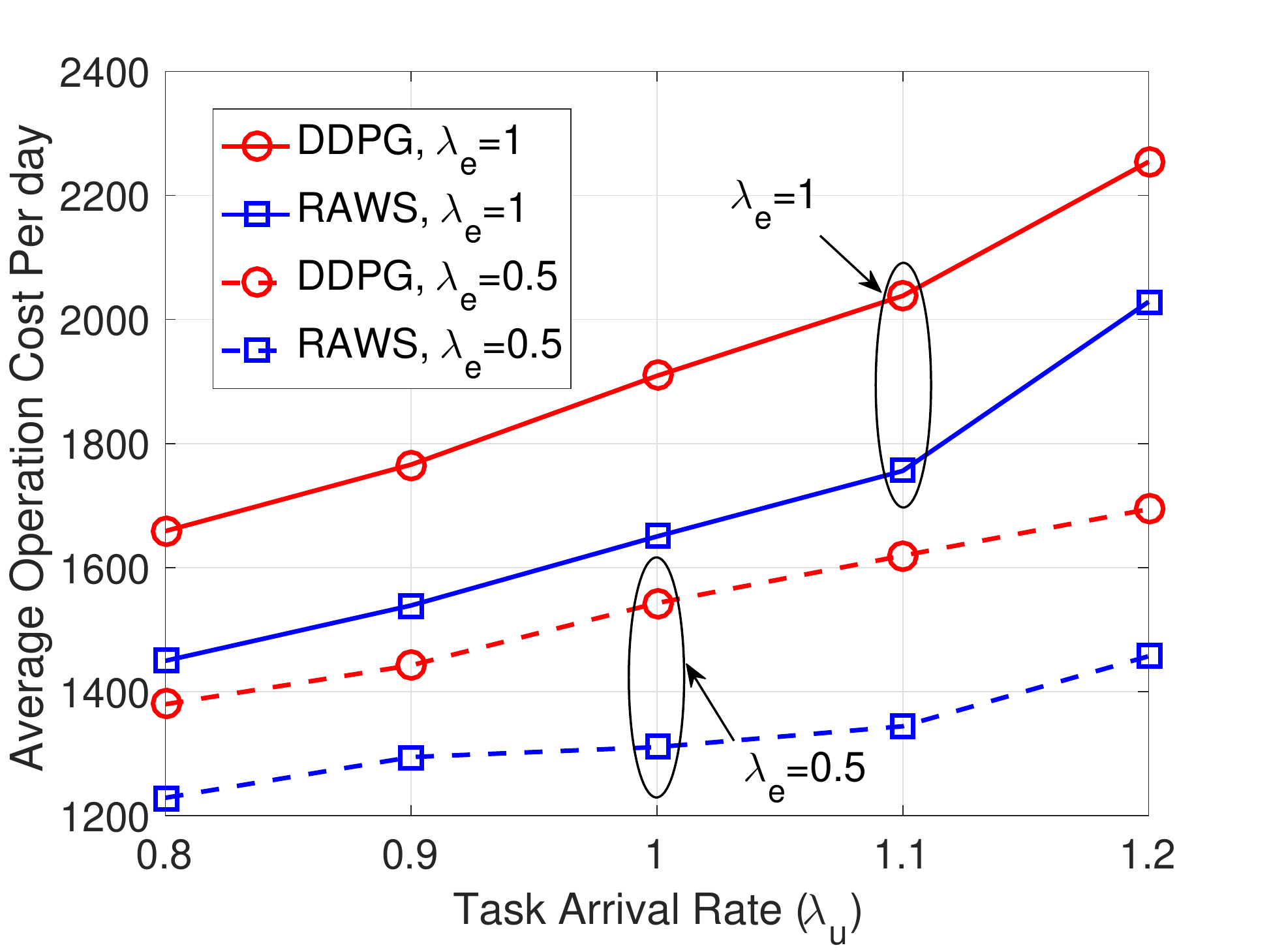}}
	\end{subfigure}
		\caption{Overall system cost and operation cost comparison with respect to the task arrival rate of the delay-sensitive service. }
	\label{Fig:transmission_performance}
	\vspace{-0.3cm}
\end{figure*}

\begin{figure}[t]
	\centering
	\renewcommand{\figurename}{Fig.}	
	\includegraphics[width=0.4\textwidth]{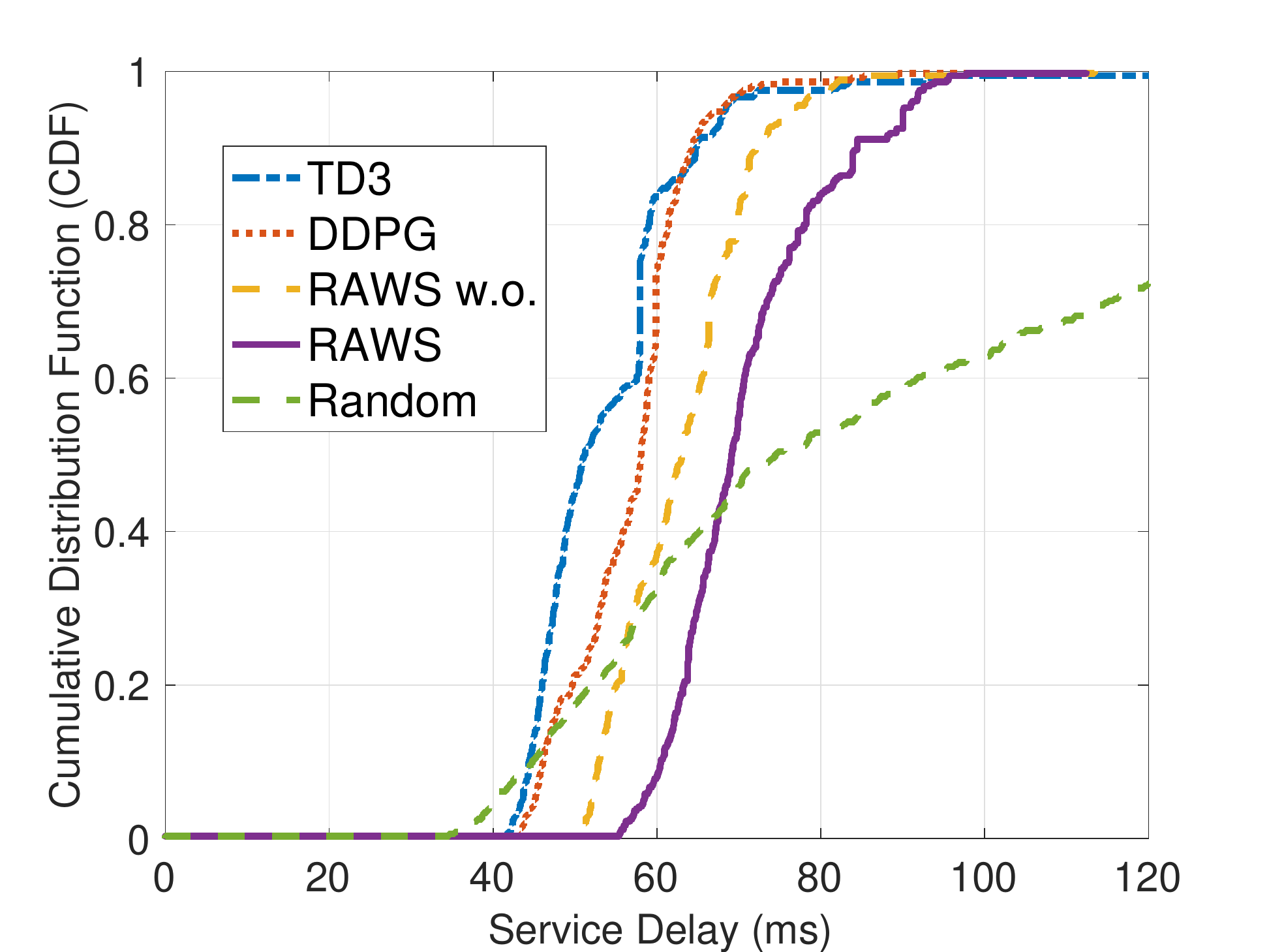}
	\caption{Cumulative distribution function of service delay for $\lambda_u=1$.}
	\label{Fig:CDF}
		\vspace{-0.3cm}
\end{figure}

\subsubsection{Service delay}
As shown in Fig.~\ref{Fig:CDF}, the cumulative distribution functions of the service delay of the different algorithms are presented. Two important observations can be obtained from the simulation results. Firstly, the average service delay of the RAWS is about 70.9\;\emph{ms}, which is larger than that of the DDPG (56.8\;\emph{ms}) and TD3 (61.8\;\emph{ms}). The average service delay of the RAWS is closer to the maximum tolerable delay constraint (100\;\emph{ms}) than those of the benchmarks, as the network resource is more efficiently utilized in the  RAWS  under the constraint. Secondly, all the learning-based algorithms, including the  DDPG, TD3 and RAWS, can effectively satisfy the service delay constraint.  More importantly, the delay bound violation probability of the RAWS is very low (0.28\%), validating the RAWS can satisfy the delay constraint with a high probability. 

\begin{table}[t]
	\footnotesize
	\centering
	\caption{{{QoS constraint violation probability with different task arrival rates.}}}
	\label{T:QoS_constraint}
	\vspace{-0.3cm}
	\begin{center}
		\begin{tabular}{| c | c | c | c | c |c |}
			\hline 
			{\textbf{$\lambda_u$ }}	& {\textbf{RAWS}} & {\textbf{DDPG}} & {\textbf{TD3}} & {\textbf{RAWS w.o.}}& {\textbf{Random}} \\
			\hline			
			{0.9} &   {0.28\% }&  {1.11\%} & {1.11\%}& {0.56\%}& 31.67\% \\ \hline
			{1.0}&  {0.28\% }&  {0.56\%}&  {0.56\%}& {1.67\%} & 35.56\% \\ \hline
			{1.1} &  {0.28\% }&   {0.28\% }&  {0.28\%}& {2.22\%} & 38.06\%\\ \hline
			{1.2} &  {0.56\% }&  {0.56\% }& {1.11\%}& {3.89}\% &40.83\% \\ \hline
			Mean & \textbf{0.35\%} & 0.63\% & 0.76\%& 2.08\% & 36.53\%\\ \hline
		\end{tabular}
	\end{center}
	\vspace{-0.3cm}
\end{table}

\subsubsection{QoS constraint violation probability}
{As listed in Table~\ref{T:QoS_constraint}, QoS constraint violation probabilities of different algorithms are compared for  task arrival rates. A QoS constraint violation event is triggered by either exceeding the maximum tolerable delay for the delay-sensitive service or violating queue stability constraints for both services. It can be seen that the average violation probability of the RAWS over task arrival rates is 0.35\%, which is lower than those of the benchmarks.  It is interesting to note that, the RAWS without workload distribution performs even worse than the learning-based benchmarks (DDPG and TD3). The workload distribution mechanism provides more flexibility to the system. It balances spatially uneven workload among BSs by distributing vehicles' workloads from overloaded BSs to underutilized BSs, thereby reducing QoS constraint violation probability.} 

\begin{figure}[t]
	\centering
	\renewcommand{\figurename}{Fig.}	
	\includegraphics[width=0.4\textwidth]{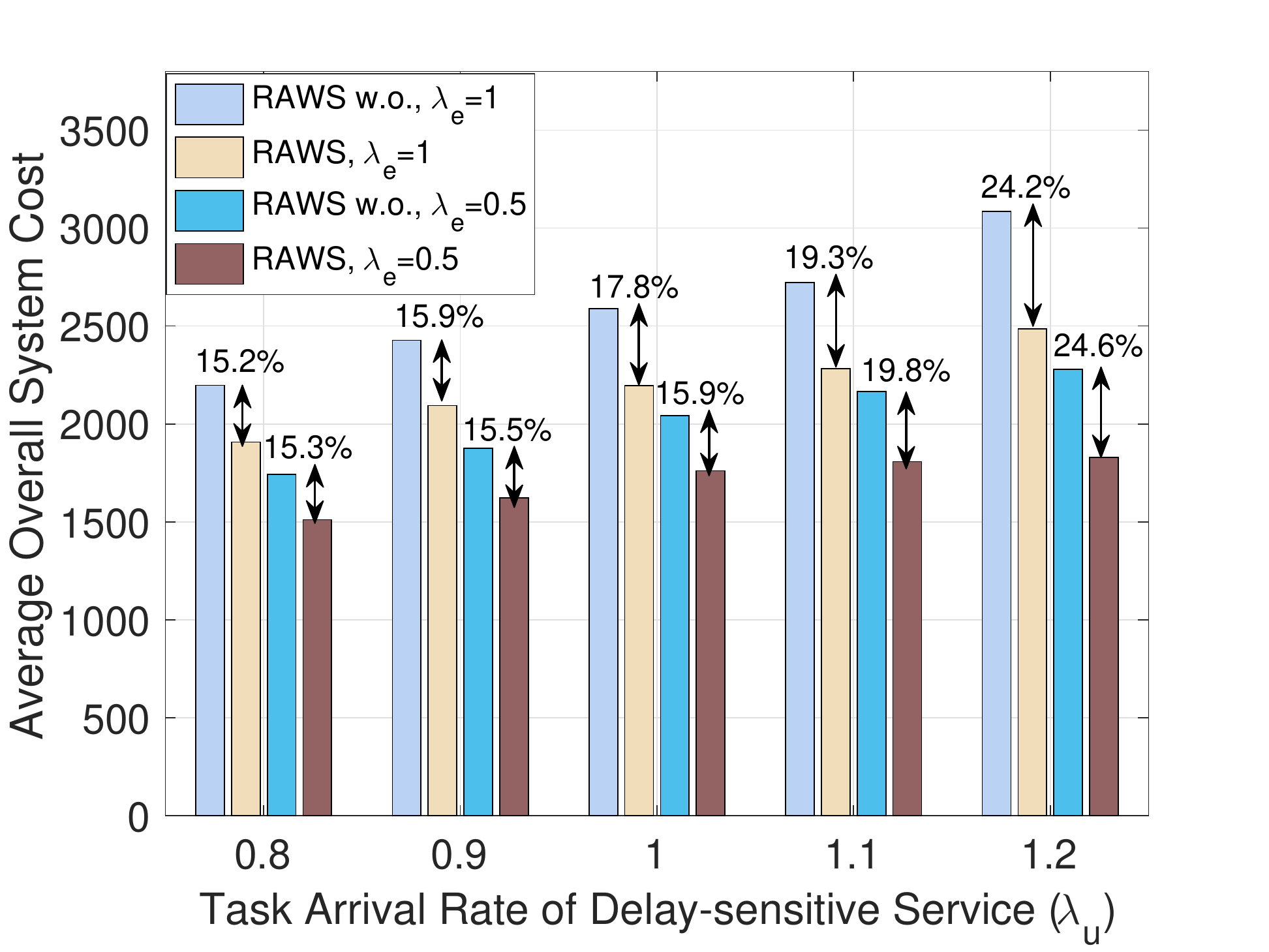}
	\caption{Impact of workload distribution on the overall system cost with respect to task arrival rate.}
	\label{Fig:workload_comparsion}
		\vspace{-0.3cm}
\end{figure}
\subsubsection{Impact of workload distribution}
As shown in Fig.~\ref{Fig:workload_comparsion}, the impact of workload distribution in terms of the system cost is evaluated. The RAWS can effectively reduce the overall system cost as compared with that without workload distribution. In addition, the performance gain increases with the increase of the task arrival rate. For example, when $\lambda_e=1$, the performance gain increases from 15.2\% for $\lambda_u=0.8$ to 24.2\% for  $\lambda_u=1.2$.  This is because the proposed algorithm with workload distribution can better utilize network resources to reduce system cost than that without workload distribution, especially in a heavy traffic scenario.

\subsection{Performance Evaluation Over  Urban Vehicle Traffic Flow Trace}
\begin{figure*}[t]
	\centering
	\renewcommand{\figurename}{Fig.}	
	\begin{subfigure}[Cumulative system cost]{
			\label{Fig:DIdi_cumulative_cost}
			\includegraphics[width=0.4\textwidth]{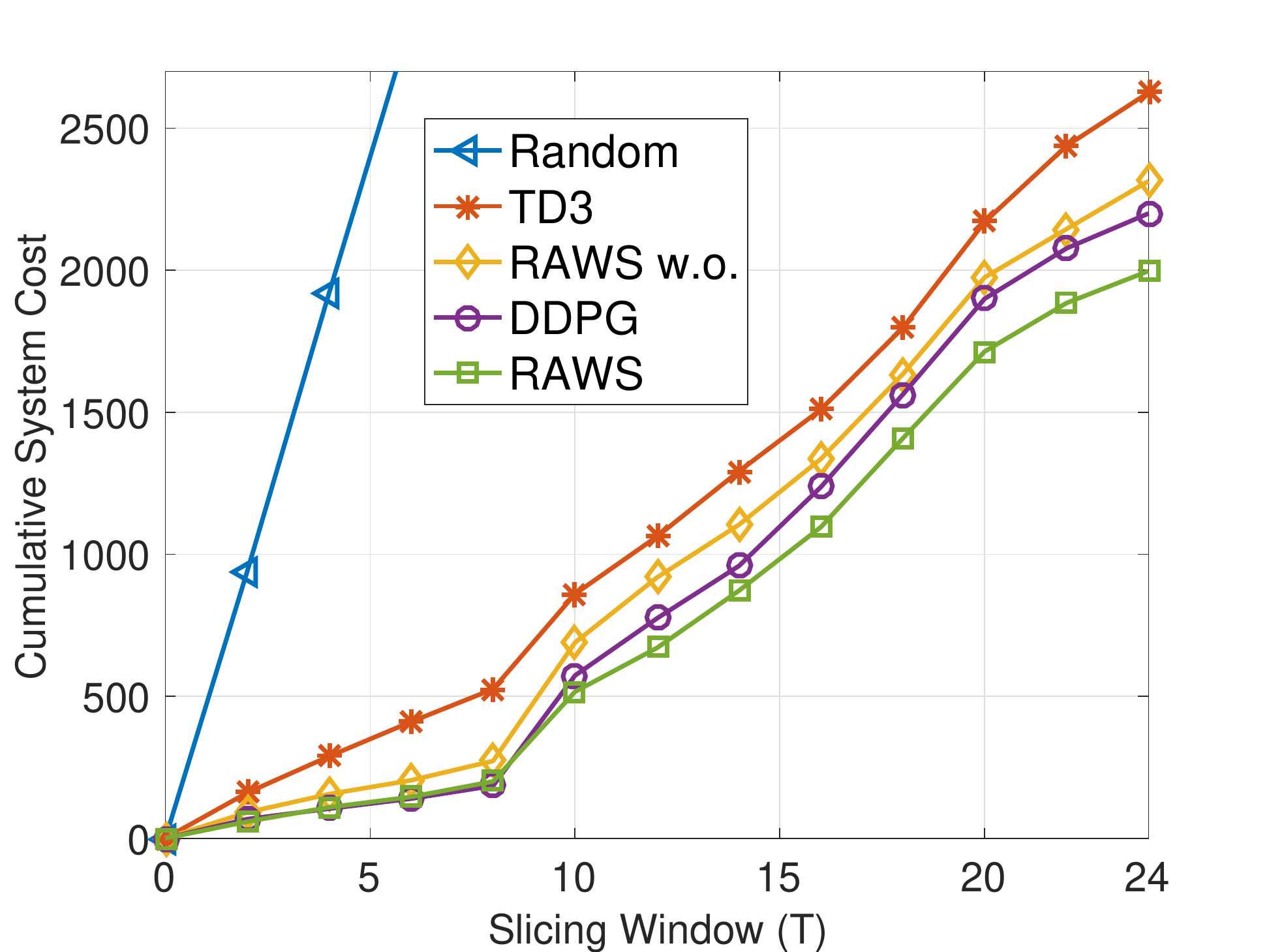}}
	\end{subfigure}
	~
	\begin{subfigure}[Service delay]{
			\label{Fig:DiDi_service_delay}
			\includegraphics[width=0.4\textwidth]{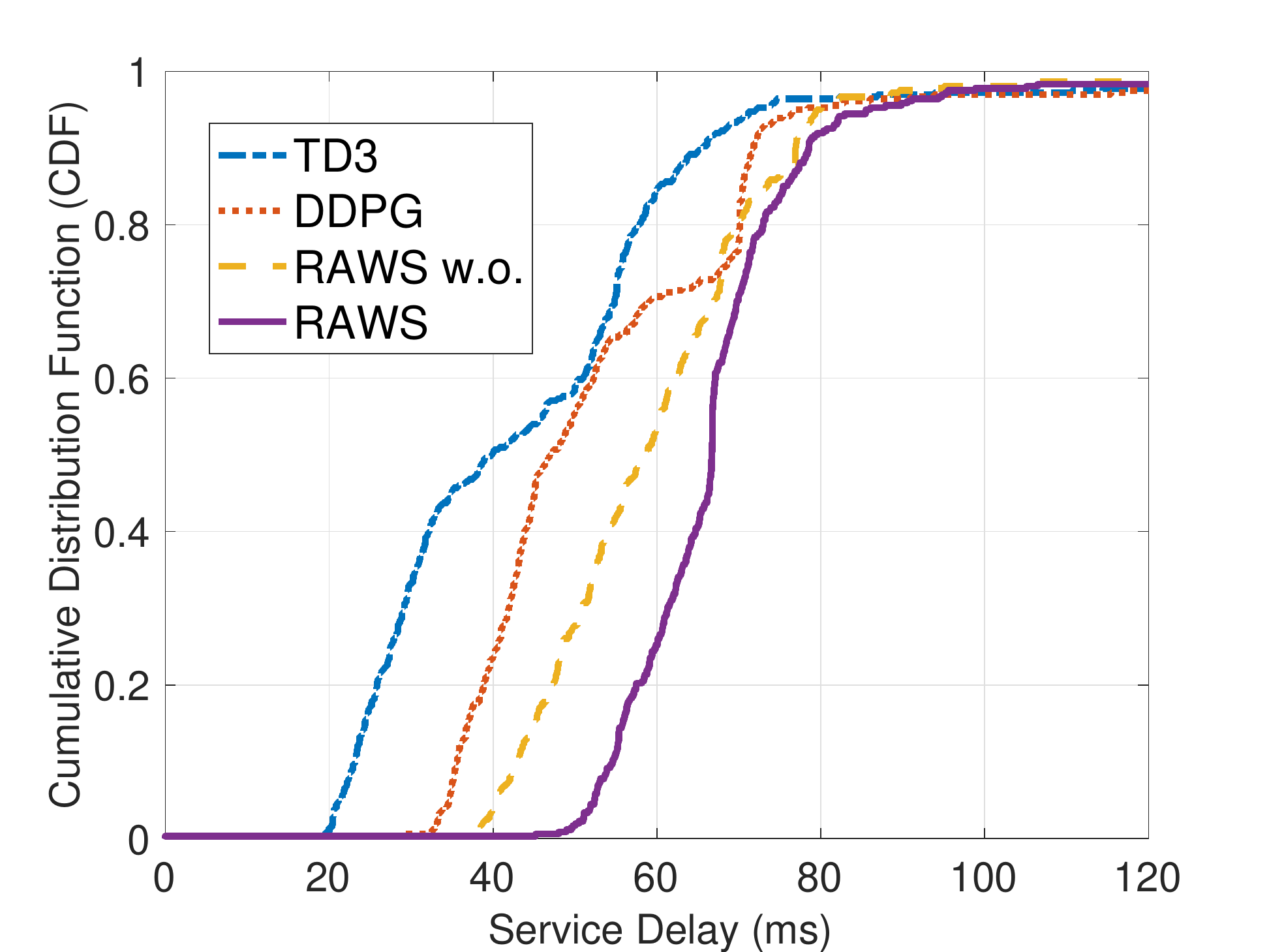}}
	\end{subfigure}
	\caption{{Cumulative system cost and service delay comparison among the different algorithms for $\lambda_u=1$ over the urban vehicle traffic flow trace.}}
	\label{Fig:DiDi_trace}
		\vspace{-0.3 cm}
\end{figure*}
In addition to the highway vehicle traffic flow trace, we further evaluate the performance of the proposed algorithm over the urban vehicle traffic flow trace, in terms of the cumulative system cost and service delay. In Fig.~\ref{Fig:DIdi_cumulative_cost}, the cumulative system cost is presented with respect to the slicing window. Similar to the simulation results over the highway trace, the proposed RAWS algorithm achieves the minimum  cumulative system cost among all the algorithms. Specifically, the RAWS reduces the cumulative system  cost within one day by approximately 9.2\% as compared with the best benchmark algorithm. In Fig.~\ref{Fig:DiDi_service_delay}, the cumulative distribution functions of the service delay of different algorithms are compared, which shows that the RAWS algorithm can guarantee the service delay constraint (100\;\emph{ms}) with a very high probability. 
\section{Conclusion}\label{sec:conclusions}
In this paper, we have presented a dynamic RAN slicing framework for vehicular networks to support diverse IoV services with different QoS requirements. We have proposed the RAWS, a two-layer constrained RL algorithm, to make the workload distribution and resource allocation decisions. Trace-driven simulation results have been provided to demonstrate that the proposed algorithm can effectively reduce the system cost, especially in the heavy traffic scenario, while satisfying QoS requirements with a high probability. The RAWS differs from traditional optimization-based methods, with the ability to adapt to time-varying vehicle traffic density without future information. In addition to RAN slicing, the design principle of the RAWS that integrates optimization and RL  can  be applied to other constrained stochastic optimization problems. 
{For the future work, we will investigate the optimal slicing window size for RAN slicing. In addition,  for the implementation in large-scale vehicular networks, we will develop a distributed and low-complexity learning based algorithm.}

\section*{Acknowledgements}
This work was financially supported by Natural Sciences and Engineering Research Council (NSERC) and Huawei Canada. The authors would like to thank the undergraduate research assistant, Anita Hu, for conducting the simulations. The authors would also like to thank Jie~Gao (Marquette University), Qihao Li, Jaya Rao, and Weisen Shi for many valuable suggestions throughout the work.

\appendix
\subsection{Proof of Theorem \ref{theorem}}\label{appendix:theorem}
	For notation simplicity, we omit $t$ in the proof. With  \eqref{equ:offloading_delay}, \eqref{equ:processing_delay} and \eqref{equ:delay_definition}, the objective function in problem $\mathbf{P}_{1,u}$ can be rewritten as 
	\begin{equation}\label{equ: revised_objective}
	\begin{split}
	f(\beta_{m,u})&=D^h_{u}+ \frac{1}{\sum_{m\in \mathcal{M}}\rho_mL\lambda_u}\cdot \\
	&\sum_{n\in \mathcal{N}}\left(\frac{\omega_{n,s,u}}{\omega_{n,s,u}-\chi_{n,u}-\sum_{m\in \mathcal{M}_o}\psi_{m,n,u}  \beta_{m,u}} \right.\\
	&+\left. \frac{\omega_{n,c,u}}{\omega_{n,c,u}-\chi_{n,u}-\sum_{m\in \mathcal{M}_o}\psi_{m,n,u}  \beta_{m,u}}-2   \right),
	\end{split}
	\end{equation}
	where $\omega_{n,s,u}=S_{n,u}{R}_n/\xi_u$ and $\omega_{n,c,u}=C_{n,u}F/\eta_u$. The second-order derivative of the objective function is given by
	\begin{equation}
	\begin{split}
	\frac{\partial^2 f(\beta_{m,u})}{\partial^2 \beta_{m,u}}
	&=\frac{1}{\sum_{m\in \mathcal{M}}\rho_mL\lambda_u} \cdot \\
	&\sum_{n\in \mathcal{N}}\left(\frac{2\omega_{n,s,u}\psi_{m,n,u}^2\beta_{m,u} }{\left(\omega_{n,s,u}-\chi_{n,u}-\sum_{m\in \mathcal{M}_o}\psi_{m,n,u}  \beta_{m,u}\right)^3}\right.\\
	&+\left. \frac{2\omega_{n,c,u}\psi_{m,n,u}^2\beta_{m,u} }{\left(\omega_{n,c,u}-\chi_{n,u}-\sum_{m\in \mathcal{M}_o}\psi_{m,n,u}  \beta_{m,u}\right)^3}\right).
	\end{split}
	\end{equation}
	Since  $\beta_{m,u} \geq 0$, we have ${\partial^2 f(\beta_{m,u})}/{\partial^2 \beta_{m,u}} \geq 0$ when the stability constraints in \eqref{equ: constraint_U_1} and \eqref{equ: constraint_U_2} are satisfied. As the constraints of the problem are linear, $\mathbf{P}_{1,u}$ is a convex optimization problem~\cite{boyd2004convex}. Hence, Theorem~\ref{theorem} is proved.

\bibliographystyle{IEEEtran}
\bibliography{security}

\vspace*{-1.5\baselineskip}
\begin{IEEEbiography}[{\includegraphics[width=1in,height=1.25in,clip,keepaspectratio]{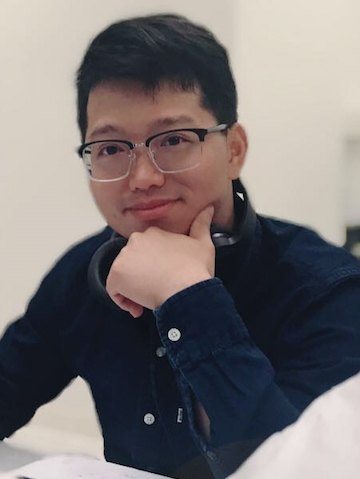}}]{Wen Wu}
	(S'13-M'20) earned the Ph.D. degree in Electrical and Computer Engineering from University of Waterloo, Waterloo, ON, Canada, in 2019. He received the B.E. degree in Information Engineering from South China University of Technology, Guangzhou, China, and the M.E. degree in Electrical Engineering from University of Science and Technology of China, Hefei, China, in 2012 and 2015, respectively. Starting from 2019, he works as a Post-doctoral fellow with the Department of Electrical and Computer Engineering, University of Waterloo. His research interests include millimeter-wave networks and AI-empowered wireless networks.
\end{IEEEbiography}

\vspace*{-1.5\baselineskip}
\begin{IEEEbiography}[{\includegraphics[width=1in,height=1.25in,clip,keepaspectratio]{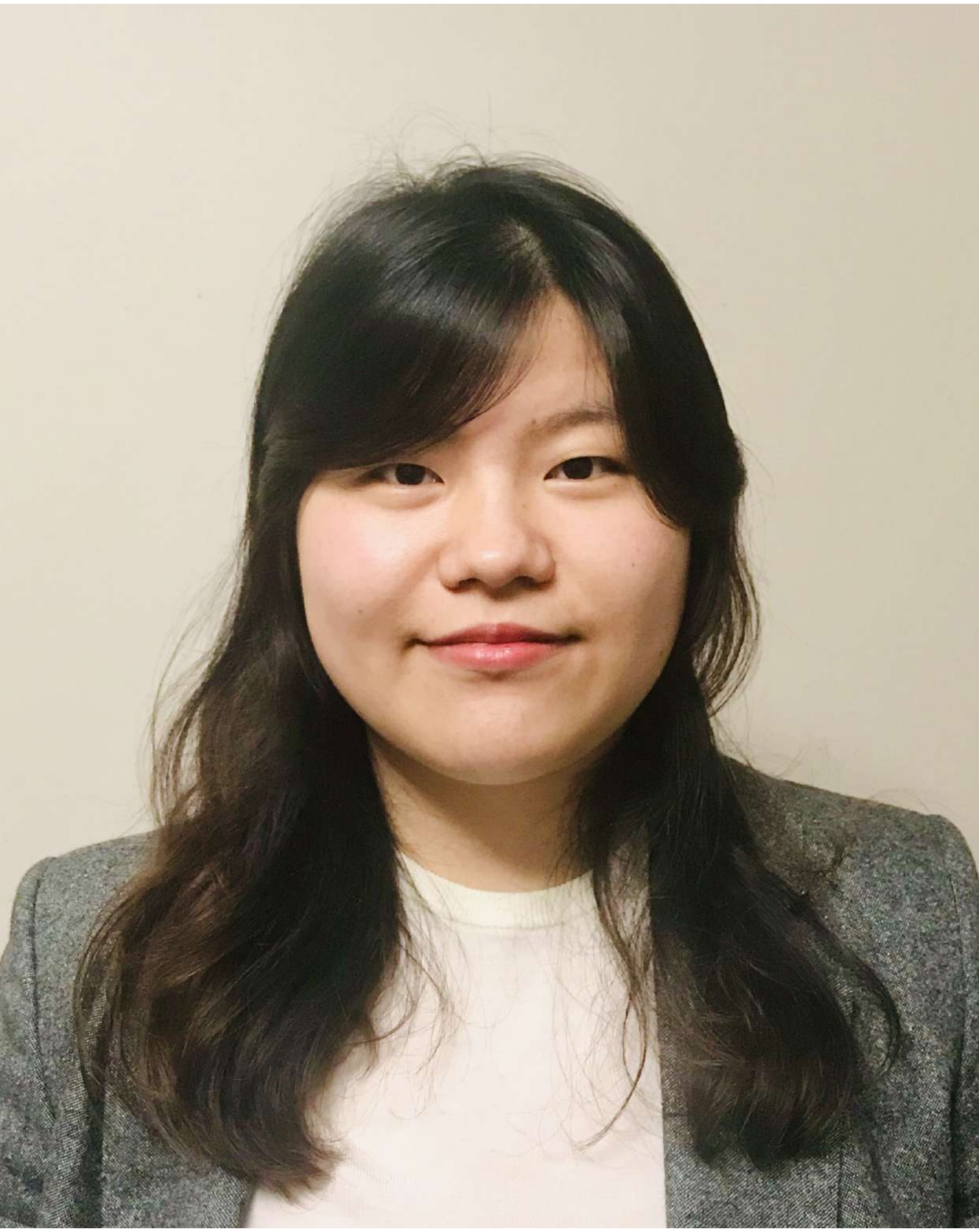}}]{Nan Chen}
	(M'20) received her Bachelor degree in Electrical Engineering and Automation from Nanjing University of Aeronautics and Astronautics, Nanjing, Jiangsu, China, in 2014, and the Ph.D. degree in Electrical and Computer Engineering from the University of Waterloo, Waterloo, ON, Canada, in 2019. Her current research interests include electric vehicle charging/discharging scheme design in smart grid, next-generation wireless networks, and	machine learning application in vehicular cyber-physical systems.
\end{IEEEbiography}

\vspace*{-1.5\baselineskip}
\begin{IEEEbiography}[{\includegraphics[width=1in,height=1.25in,clip,keepaspectratio]{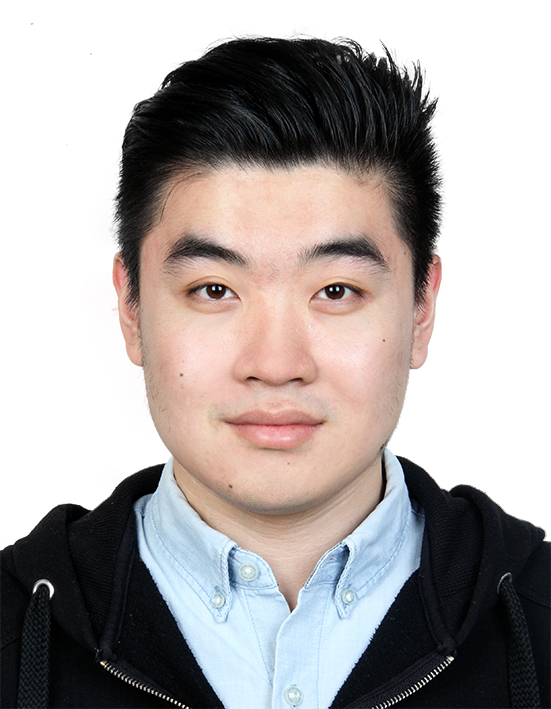}}]{Conghao Zhou}
	(S'19) received the B.S. degree from Northeastern University, Shenyang, China, in 2017 and received the M.S. degree from University of Illinois at Chicago, Chicago, IL, USA, in 2018. He is currently working toward the Ph.D. degree with the Department of Electrical and Computer Engineering, University of Waterloo, Waterloo, ON, Canada. His research interests include space-air-ground integration networks and machine learning in wireless networks.
\end{IEEEbiography}

\vspace*{-1.5\baselineskip}
\begin{IEEEbiography}[{\includegraphics[width=1in,height=1.25in,clip,keepaspectratio]{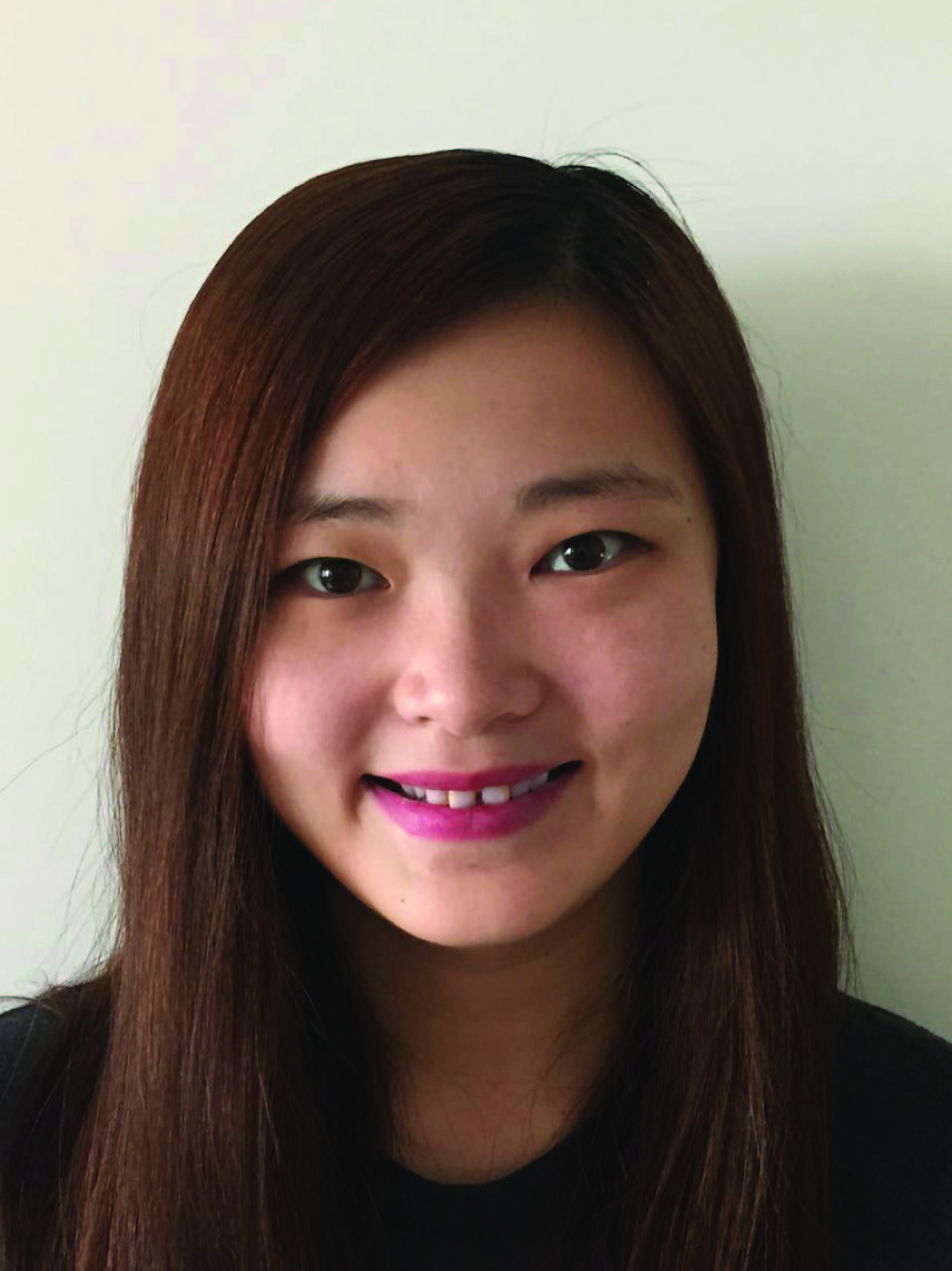}}]{Mushu Li}(S'18) received the B.Eng. degree from the University of Ontario Institute of Technology (UOIT), Canada, in 2015, and the M.A.Sc. degree from Ryerson University, Canada, in 2017. She is currently working toward the Ph.D. degree in electrical engineering at University of Waterloo, Canada. She was the recipient of Natural Science and Engineering Research Council of Canada Graduate Scholarship (NSERC-CGS) in 2018, and Ontario Graduate Scholarship (OGS) in 2015 and 2016, respectively. Her research interests include the system optimization in VANETs and machine learning in wireless networks.
\end{IEEEbiography}

\vspace*{-1.5\baselineskip}
\begin{IEEEbiography}[{\includegraphics[width=1in,height=1.25in,clip,keepaspectratio]{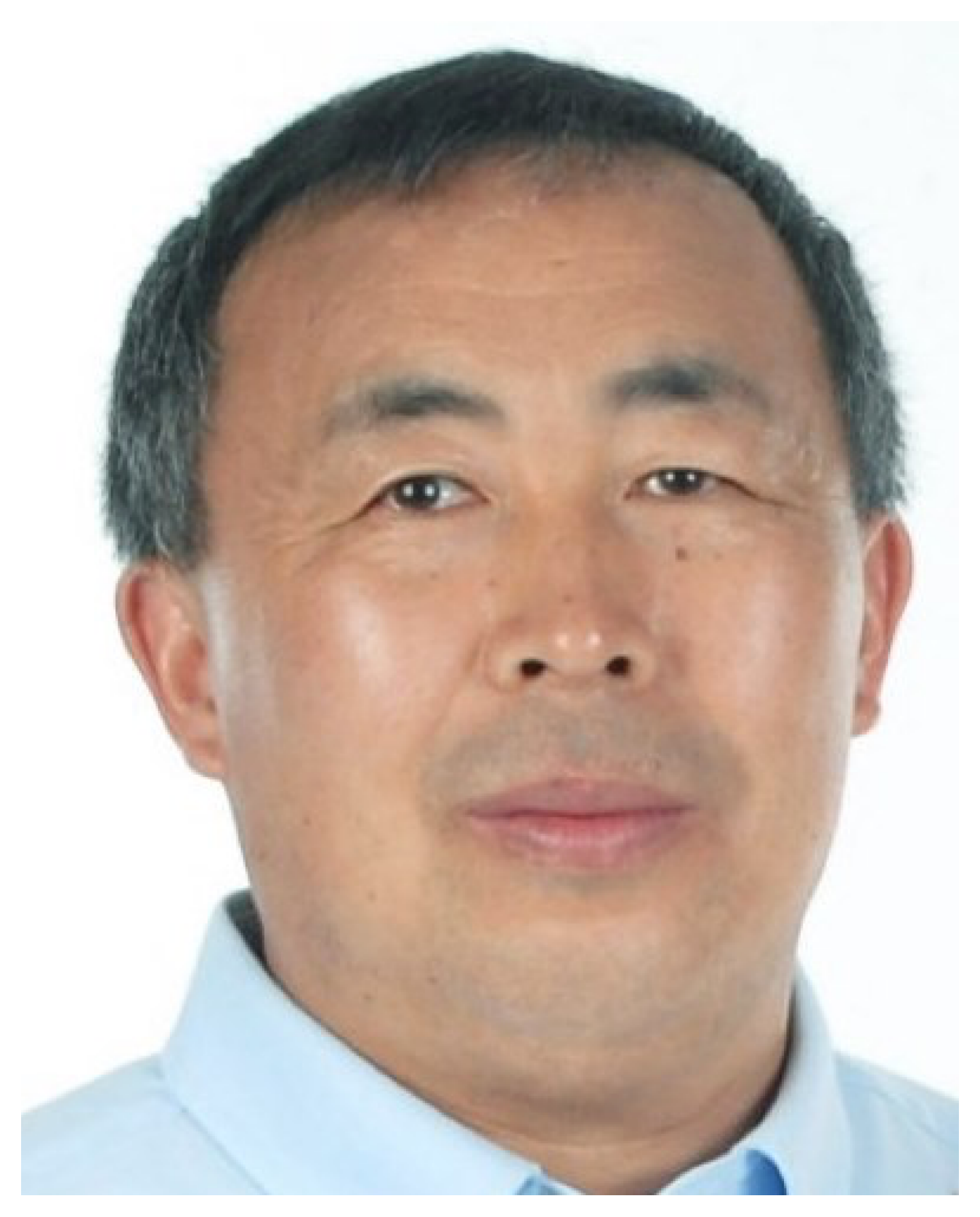}}]{Xuemin (Sherman) Shen}(M'97-SM'02-F'09) received the Ph.D. degree in electrical engineering from Rutgers University, New Brunswick, NJ, USA, in 1990. He is currently a University Professor with the Department of Electrical and Computer Engineering, University of Waterloo, Canada. His research focuses on network resource management, wireless network security, Internet of Things, 5G and beyond, and vehicular ad hoc and sensor networks. Dr. Shen is a registered Professional Engineer of Ontario, Canada, an Engineering Institute of Canada Fellow, a Canadian Academy of Engineering Fellow, a Royal Society of Canada Fellow, a Chinese Academy of Engineering Foreign Member, and a Distinguished Lecturer of the IEEE Vehicular Technology Society and Communications Society. 
	
Dr. Shen received the R.A. Fessenden Award in 2019 from IEEE, Canada, Award of Merit from the Federation of Chinese Canadian Professionals (Ontario) in 2019, James Evans Avant Garde Award in 2018 from the IEEE Vehicular Technology Society, Joseph LoCicero Award in 2015 and Education Award in 2017 from the IEEE Communications Society, and Technical Recognition Award from Wireless Communications Technical Committee (2019) and AHSN Technical Committee (2013). He has also received the Excellent Graduate Supervision Award in 2006 from the University of Waterloo and the Premier’s Research Excellence Award (PREA) in 2003 from the Province of Ontario, Canada. He served as the Technical Program Committee Chair/Co-Chair for IEEE Globecom’16, IEEE Infocom’14, IEEE VTC’10 Fall, IEEE Globecom’07, and the Chair for the IEEE Communications Society Technical Committee on Wireless Communications. Dr. Shen is the elected IEEE Communications Society Vice President for Technical \& Educational Activities, Vice President for Publications, Member-at-Large on the Board of Governors, Chair of the Distinguished Lecturer Selection Committee, Member of IEEE ComSoc Fellow Selection Committee. He was/is the Editor-in-Chief of the IEEE IoT JOURNAL, IEEE Network, IET Communications, and Peer-to-Peer Networking and Applications. 
	
\end{IEEEbiography}

\begin{IEEEbiography}[{\includegraphics[width=1in,height=1.25in,clip,keepaspectratio]{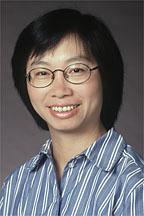}}]{Weihua Zhuang}
	(M'93-SM'01-F'08) has been with the Department of Electrical and Computer Engineering, University of Waterloo, Waterloo, ON, Canada, since 1993, where she is a Professor and a Tier I Canada Research Chair in Wireless Communication Networks.
	
	Dr. Zhuang was a recipient of the 2017 Technical Recognition Award from the IEEE Communications Society Ad Hoc and Sensor Networks Technical Committee, and a co-recipient of several Best Paper Awards from IEEE conferences. She was the Editor-in-Chief of the IEEE TRANSACTIONS ON VEHICULAR TECHNOLOGY from 2007 to 2013, the Technical Program Chair/Co-Chair of IEEE VTC Fall 2017 and Fall 2016, and the Technical Program Symposia Chair of IEEE Globecom 2011. She is an elected member of the Board of Governors and Vice President - Publications of the IEEE Vehicular Technology Society. She was an IEEE Communications Society Distinguished Lecturer from 2008 to 2011. Dr. Zhuang is a Fellow of the Royal Society of Canada, the Canadian Academy of Engineering, and the Engineering Institute of Canada.
	
\end{IEEEbiography}

\begin{IEEEbiography}[{\includegraphics[width=1in,height=1.25in,clip,keepaspectratio]{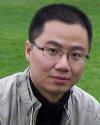}}]{Xu Li} is a senior principal researcher at Huawei Technologies Canada. He received a PhD (2008) degree from Carleton University, an MSc (2005) degree from the University of Ottawa, and a BSc (1998) degree from Jilin University, China, all in computer science. Prior to joining Huawei, he worked as a research scientist (with tenure) at Inria, France. His current research interests are focused in 5G. He contributed extensively to the development of 3GPP 5G standards through 90+ standard proposals. He has published 100+ refereed scientific papers and is holding 30+ issued US patents. He is/was on the editorial boards of IEEE Communications Magazine, the IEEE Transactions on Parallel and Distributed Systems, the Wiley Transactions on Emerging Telecommunications Technologies and a number of other international archive journals.
\end{IEEEbiography}

\end{document}